\documentclass[letterpaper, 10 pt, conference]{ieeeconf}  

\IEEEoverridecommandlockouts                              

\usepackage{amsmath} 
\usepackage{flushend}

\usepackage{amssymb}  
\usepackage{graphicx}
\usepackage{subcaption}
\usepackage{float}
\usepackage{placeins}

\usepackage[colorlinks=true, citecolor=magenta]{hyperref}

\newcommand{\framework}{\textsc{LivePoint}}

\usepackage{color}
\usepackage[dvipsnames]{xcolor}
\usepackage{colortbl}
\usepackage{xcolor}
\usepackage{mathrsfs}
\usepackage{tikz}
\usetikzlibrary{positioning}
\usepackage{xspace}
\usepackage{multicol,multirow}

\usepackage{amsthm}
\usepackage[outline]{contour}
\usepackage{soul}
\usepackage[]{mdframed}
\usepackage{threeparttable}
\usepackage{pifont}
\usepackage{wrapfig}
\usepackage{amsmath, amsfonts, amssymb}   
\usepackage{mathtools}
\usepackage{pifont}%
\let\labelindent\relax
\usepackage{enumitem}
\usepackage{booktabs}
\usepackage{float}
\usepackage[linesnumbered,ruled,vlined]{algorithm2e}
\usepackage{bm}
\usepackage{subcaption}
\usepackage{makecell}
\usepackage[font=small]{caption}
\newtheorem{theorem}{Theorem}

\newtheorem{definition}{Definition}

\usepackage[authormarkup=none]{changes}
\setaddedmarkup{{\color{blue!75!black}#1}}
\setdeletedmarkup{\protect{\color{blue!25!gray}\sout{#1}}}
\definechangesauthor[name={Chris}]{CH}

\title{\LARGE \bf

\framework{}: Fully Decentralized, Safe, Deadlock-Free Multi-Robot Control in Cluttered Environments with High-Dimensional Inputs

}

\author{Jeffrey Chen, Rohan Chandra \\
\href{https://livepoint-uva.github.io/}{\small Code, Videos, Proofs at \textbf{livepoint-uva.github.io}} \\
\thanks{Authors are with the Department of Computer Science, University of Virginia. {\tt\small \{fyy2ws, aar8xx\}@virginia.edu}} 
}

\begin{document}

\maketitle

\thispagestyle{empty}
\pagestyle{empty}



\begin{abstract}

Fully decentralized, safe, and deadlock-free multi-robot navigation in dynamic, cluttered environments is a critical challenge in robotics. Current methods require exact state measurements in order to enforce safety and liveness \textit{e.g.} via control barrier functions (CBFs), which is challenging to achieve directly from onboard sensors like lidars and cameras. This work introduces \framework, a decentralized control framework that synthesizes universal CBFs over point clouds to enable safe, deadlock-free real-time multi-robot navigation in dynamic, cluttered environments. Further, \framework{} ensures minimally invasive deadlock avoidance behavior by dynamically adjusting agents' speeds based on a novel symmetric interaction metric. We validate our approach in simulation experiments across highly constrained multi-robot scenarios like doorways and intersections. Results demonstrate that \framework{} achieves zero collisions or deadlocks and a $100\%$ success rate in challenging settings compared to optimization-based baselines such as MPC and ORCA and neural methods such as MPNet, which fail in such environments. Despite prioritizing safety and liveness, \framework{} is $35\%$ smoother than baselines in the doorway environment, and maintains agility in constrained environments while still being safe and deadlock-free.




\end{abstract}


\section{INTRODUCTION}


The dream of having robots work with us in our kitchens, construction sites, and hospitals has driven interest in multi-robot navigation among autonomous vehicles~\cite{chandra2022towards, parikh2024transfer, chandra2020cmetric, chandra2020graphrqi, chandra2021using}, warehouse robots~\cite{chandra2023socialmapf}, and personal home robots~\cite{sharma2023review}. Often, these applications feature small, cluttered environments (such as doorways, hallways, or rooms filled with obstacles). Humans naturally and gracefully navigate these environments every day, such as by weaving through a crowd or slowing down \textit{by just enough} to let another reach the doorway first~\cite{raj2024rethinking, chandra2024towards, francis2023principles}. For robots to seamlessly navigate through these cluttered environments in a similar manner, they must learn to navigate like humans--safely, gracefully, and without getting stuck (deadlock-free). 

Conventional wisdom~\cite{margolis2021,chandra2024_socialgames,gouru2024,zinage2024} tells us that in order for robots to achieve human-like mobility in cluttered environments, their low-level controllers need exact and accurate state measurements of their surroundings which is difficult in practice to realize, especially if the environment is dynamic. Most roboticists, however, would ideally prefer navigation systems that produce human-like trajectories directly using input from onboard sensors such as lidars and cameras, without relying on expensive mapping and perception for exact state measurements~\cite{desa2024_pointcloud}. For instance, Sa et al.~\cite{desa2024_pointcloud} perform point cloud-based single robot navigation to handle dynamic environments, allowing robots to react to rapidly changing obstacles. 


\begin{figure}[t]
      \centering
      \includegraphics[scale=0.47]{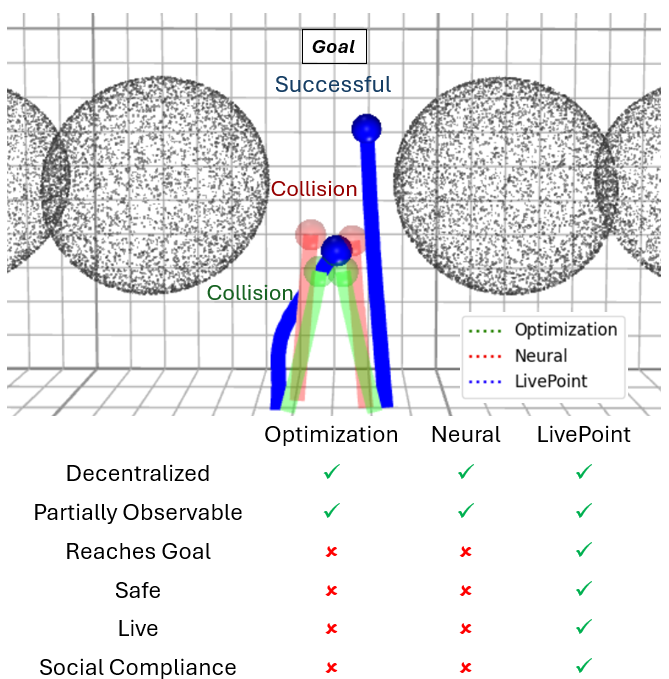}       \caption{Our decentralized approach enables effective, safe, live, and socially compliant robot navigation in cluttered environments using only high-dimensional point cloud input.} 
      \label{cover_figure}
      \vspace{-15pt}
   \end{figure}


There are several challenges to safe and deadlock-free multi-robot navigation in cluttered environments with high-dimensional inputs such as point clouds. The first key challenge is that ensuring both safety and liveness using dense point clouds can be complex \cite{desa2024_pointcloud, guan2022m3detr} and thus, computationally expensive. For instance, some analytical methods use control barrier functions (CBFs)~\cite{desa2024_pointcloud,tong2023_nerf} to guarantee safety, which requires intensive computations to evaluate the barrier function and its derivatives. Additionally, in decentralized systems, we have no central authority that can coordinate agents in a manner that deadlocks will be prevented or resolved, particularly critical when agents are self-interested (each optimizes its own individual objective function) and have conflicting objectives~\cite{chandra2023socialmapf, chandra2022game, chandra2022game, suriyarachchi2022gameopt, suriyarachchi2024gameopt+}. Lastly, learning-based methods~\cite{pointcloudbasedreinforcement,cosner2022_stereovision,dawson2022learningsafegeneralizableperceptionbased,xiao2022differentiablecontrolbarrierfunctions, mavrogiannis2022b} can struggle with generalization and lack formal safety guarantees. 

Moreover, cluttered environments pose additional challenges in computing safe paths for multiple robots, particularly due to geometric symmetry in the environment~\cite{arul2023dsmpepcsafedeadlockavoidingrobot}. Symmetry arises when multiple robots find themselves in nearly identical situations--having similar distances to their goals, nearly equal velocities, and overlapping optimal trajectories\cite{grover2020_symmetry}. Unlike humans, who naturally resolve such conflicts through social cues and subtle motion adjustments, robots lack an inherent mechanism to break symmetry, making coordination difficult and leading to deadlocks\cite{grover2016_deadlock}. Geometric and model-based methods \cite{zhou2017,Fiorini1998MotionPI,vanderberg2008} have been explored to address navigation in such environments, but conflicting navigation objectives can still result in deadlock~\cite{arul2023dsmpepcsafedeadlockavoidingrobot}. Reactive collision avoidance techniques, such as Optimal Reciprocal Collision Avoidance (ORCA)~\cite{orca2011} attempt to resolve these issues by computing admissible velocity sets. However, in highly constrained environments, these methods can fail, as robots overly adjust their velocities to avoid collisions but ultimately stall~\cite{dergachev2021distributedmultiagentnavigationbased}. Many existing approaches remain too conservative to effectively navigate cluttered environments without deadlock \cite{arul2023dsmpepcsafedeadlockavoidingrobot}. 



\subsection{Main Contributions}

To address these challenges, we propose a new multi-robot navigation approach that takes as input point clouds and produces in real-time safe and deadlock-free trajectories for fully decentralized robots in cluttered environments. Key to our approach is a novel \textit{universal CBF} formulation that dynamically and simultaneously computes safety and liveness certificates for the controller. The safety component of the universal CBF follows the standard CBF definition commonly employed in the literature~\cite{ames2016_cbf}. The liveness component is driven by an interaction function that intelligently breaks deadlock-causing symmetry between agents in a minimally invasive fashion, and is cast as a CBF. In summary, our contributions are as follows:

\begin{enumerate}
    \item  \textbf{A fully decentralized multi-robot navigation algorithm with point cloud input:} We introduce a sequential control strategy that enables each agent to dynamically adjust its safety constraints based on real-time observations of other agents. Each robot processes point cloud data to construct CBFs while considering the motion of dynamic obstacles (i.e., other agents) in a real time, iterative framework, where each control cycle takes $0.04$ seconds. 
    
    \item \textbf{Synthesizing a universal CBF for safety and liveness:} We introduce a deadlock prevention mechanism that quantifies symmetry between interacting agents, identifying cases where robots are likely to deadlock, particularly in highly constrained spaces like doorways and intersections. When symmetry is detected, we apply minimal velocity perturbations to proactively break deadlock-prone configurations while maintaining safety guarantees. By integrating the deadlock detection mechanism within the same framework as the CBF constraints, \framework{} unifies CBF-based safety and deadlock avoidance into a single cohesive approach, eliminating the need for separate heuristics or rule-based strategies. Our universal CBF ensures that each robot not only avoids collisions but also maintains continuous progress, even in highly constrained environments.

\end{enumerate}

We evaluate our proposed multi-agent Depth-CBF framework in simulation scenarios, including highly constrained environments such as narrow doorways and intersections. Results show that \framework{} reduces collisions and deadlocks to
zero and achieves 100\% successful navigation in challenging
settings where previous methods fail.


\subsection{Organization of Paper}
Section \ref{Related Work} presents relevant prior work that has been done in the areas relevant to our work.  Section \ref{Problem Formulation} introduces the problem formulation, including the notation used throughout the paper, general notions of multi-agent navigation, and the concept of Social Mini Games. Section \ref{Background} presents important background on CBFs and Depth-CBF. Section \ref{Approach} outlines the design of our proposed approach, detailing the expansion to multi-robot navigation and the measures implemented to prevent potential deadlocks. Section \ref{Experiments} provides a comprehensive analysis of our simulation results, including comparisons with relevant baselines. We conclude with Section \ref{Conclusion}, providing a summary of key findings and potential directions for future research.

\section{Related Work}
\label{Related Work}
\subsection{Vision-Based Navigation}

Traditional vision-based navigation often relies on simultaneous localization and mapping (SLAM)-based approaches, such as occupancy grid methods \cite{elfes1989_occupancy}, to construct global maps for control and planning. While effective in structured environments, these methods are computationally intensive and struggle to adapt in dynamic settings. More recent methods address some of the limitations with visual-based approaches that combine planning and learning~\cite{chaplot2020learningexploreusingactive,faust2018prmrllongrangeroboticnavigation,chiang2019learningnavigationbehaviorsendtoend,meng2020scalinglocalcontrollargescale}. These methods utilize learning to solve short-horizon tasks and planning over non-metric maps to reason over longer horizons~\cite{Shah_2021}. While such approaches eliminate the need to do detailed map building, they make multiple imperfect assumptions, such as guaranteed access to a perfect replica of the environment~\cite{meng2020scalinglocalcontrollargescale} or requiring online data collection~\cite{chaplot2020learningexploreusingactive}. Control Barrier Functions (CBFs) have also been explored as an alternative to enforce safety constraints without full mapping~\cite{ames2019_cbf}, with extensions incorporating learning-based methods~\cite{cosner2022_stereovision} and and Neural Radiance Fields (NeRF)-based representations~\cite{tong2023_nerf}. However, these approaches remain limited by by high computational costs and are generally limited to static environments. Depth-CBF \cite{desa2024_pointcloud} introduced direct point-cloud based safe navigation, offering better computational efficiency and adaptability, but does not account for multi-robot navigation or coordination. 

\subsection{Safety and Liveness in Cluttered Environments}
 
Deadlocks commonly arise in multi-agent navigation due to symmetry in the environment or trajectory conflicts \cite{grover2016_deadlock, grover2020_symmetry, chandra2024_socialgames, zinage2024, gouru2024livenet}. Stochastic perturbations offer a straightforward resolution but often suboptimal trajectories  \cite{wang2017_safetybarrier}. Structured approaches include rule-based strategies such as the right-hand rule \cite{chen2022_mpc} or clockwise rotation \cite{grover2016_deadlock} improve performance, but their reliance on predetermined ordering makes them less generalizable. Other scheduling and priority-based frameworks for deadlock resolution are inspired by intersection management literature \cite{zhong2020_intersection}, such as first-come, first-served (FCFS) \cite{au2015_intersection}. Such approaches can lead to inefficiencies and congestion, especially when multiple agents arrive simultaneously. Auction-based systems \cite{carlino2013_auction} offer more flexibility by dynamically assigning priorities based on bidding but introduce communication overhead. The concept of deadlock prevention, rather than resolution, aims to predict and mitigate potential deadlocks before they even occur. Utilizing long-horizon planning allows for potential preemptive perturbation of the robots in order to avoid subsequent deadlocks \cite{chandra2024_socialgames}.

To bridge the gaps in point-cloud based navigation and deadlock prevention, we introduce \framework{}, which efficiently uses point cloud data for safe multi-robot navigation, while incorporating liveness-driven robot velocity perturbation mechanisms for deadlock-free multi-agent navigation. This approach retains the computational efficiency and adaptability while addressing the unique coordination challenges that arise in multi-agent settings.

\section{PROBLEM FORMULATION}
\label{Problem Formulation}

We address the problem of safe and deadlock-free navigation for multiple robots in a cluttered environment, using point cloud data. In this section, we define the problem as well as mathematically formalize cluttered environments as social mini-games, and state our overall objective. 



We first formulate a \textit{general} multi-agent navigation scenario using the following partially observable stochastic game (POSG)~\cite{hansen2004_dynamicprogramming}: $\left \langle k, X, U^i, \mathcal{T}, \mathcal{J}^i, O^i, \Omega^i \right \rangle$ . A superscript of $i$ refers to the $i^\textrm{th}$ agent, where $i \in [1, .., k]$ and a subscript of $t$ refers to discrete time step $t$ where $t \in [1,..,T]$. At any given time step $t$, agent $i$ has state $x^i_t \in X$. The dynamics of the agents are defined by the transition function $\mathcal{T} : X \times U^i \to X$ at time $t \in [1, .., T]$. The cost function, $J^i : X \times U^i \to \mathbb{R}$ is used to evaluate a specified control action for the agent’s current state. The terminal cost $\mathcal{J}^i : X \to \mathbb{R}$ is used to calculate the cost of the terminal state $x^i_T$. Each agent $i$ also has an observation $o^i_t \in \Omega^i$ which is determined via the observation function $o^i_t = O^i(x^i_t,P_t^i)$. This observation function takes in the state of the robot, as well as its perceived point cloud $P_t^i$, as input, and outputs the observations that the robot makes. A discrete trajectory of agent $i$ is defined by $\Gamma^i = \left(x^i_0, x^i_1, ..., x^i_T\right)$, and has a corresponding control input sequence $\Psi^i = \left(u^i_0, ..., u^i_{T -1}\right)$. Agents follow discrete and deterministic control-affine dynamics given by $x^i_{t+1} = f(x^i_t) + g(x^i_t)u^i_t$, where $f, g$ are locally Lipschitz continuous functions. At any time $t$, each agent $i$ occupies a space given by $C^i(x^i_t) \subseteq X$. Robots $i$ and $j$ are considered to have collided at time $t$ if $C^i(x^i_t) \cap C^j(x^j_t) \neq \emptyset$.

\subsection{Social Mini Games}
   

A Social Mini-Game (SMG) is a particular type of a POSG and formally characterizes a cluttered environment. It is defined as follows:

\begin{definition}
    A social mini-game occurs if for some $\delta > 0$ and integers $a, b \in (0, T)$ with $b - a > \delta$, there exists at least one pair $i, j$, $i \neq j$ such that for all $\Gamma^i \in \tilde\Gamma^i$, $\Gamma^j \in \tilde\Gamma^j$, we have 
{\small\[
C^i(x^i_t) \cap C^j(x^j_t) \neq \emptyset \quad \forall t \in [a, b],
\]}
where $x^i_t, x^j_t$ are elements of $\Gamma^i$ and $\Gamma^j$, and $\tilde\Gamma^i$ is the trajectory robot $i$ would take in the absence of any other robots.
\end{definition}
In other words, a SMG occurs when at least two agents have conflicting preferred trajectories over a sustained time interval. That is, if each agent were to follow its natural, unimpeded path ($\tilde\Gamma^i$), their occupied regions $C^i(x^i_t)$ would overlap for at least $b-a$ seconds. This overlap means that the agents must coordinate their movements, lest they end up colliding. Such a definition captures how symmetry in the environment contributes to SMGs. When two or more agents are in a nearly symmetrical scenario, having similar distances to their respective goals and nearly identical velocities. they may arrive at areas like doorways or intersections simultaneously. Figure~\ref{smg} shows an example of an SMG. 
\begin{figure}[t]
      \centering
      \includegraphics[scale=0.3]{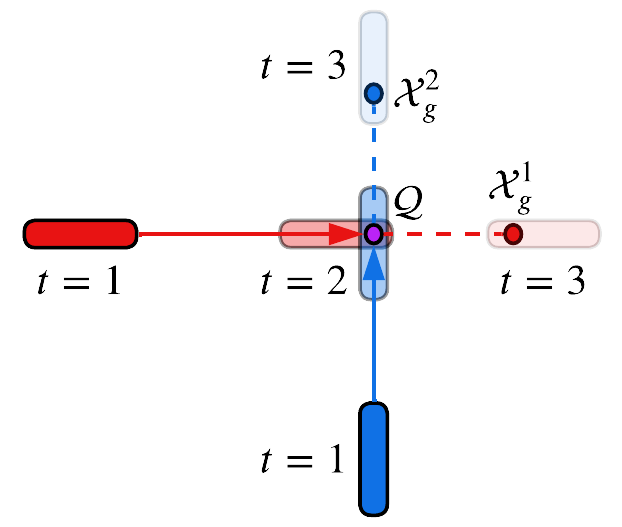}
      \caption{Example of a Social Mini-Game. The two robots (red and blue) have straight-line preferred trajectories  that intersect at $t=2$, as both must pass through $\mathcal{Q}$ on the way to their goal states, $X^1_g$ and $X^2_g)$. Without predefined rules or corrective actions, the two robots will collide.} 
      \label{smg}
      \vspace{-10pt}

   \end{figure}


The goal of our approach is as follows:
 
\textbf{Objective}:The optimal trajectory ($\Gamma^{i,*}$) and corresponding optimal sequence of control inputs ($\Psi^{i,*}$) for the $i^\textrm{th}$ robot are defined as those that minimize the following cost function:

{\small\begin{subequations} \label{eq:optimal_control}
    \begin{align}
        (\Gamma^{i,*}, \Psi^{i,*}) &= \arg\min_{\Gamma^i, \Psi^i} \sum_{t=0}^{T-1} \mathcal{J}^i(x^i_t, u^i_t) + \mathcal{J}^i_T(x^i_T) \label{eq:cost_function} \\
        \text{subject to:} \quad & \nonumber \\
        & C^i(x^i_t) \cap C^j(x^j_t) = \emptyset, \quad \forall i \neq j, \quad \forall t \in [0, T] \quad &&  \label{eq:collision_avoidance} \\
        &\| u_t^i \| > 0, \quad \forall i \in \{1, ..., k\}, \quad \forall t \in [0, T-1] \quad && \label{eq:deadlock_prevention} \\
        & x^i_{t+1} = f(x^i_t) + g(x^i_t)u^i_t, \quad \forall t \in [0, T-1] \quad &&  \label{eq:system_dynamics} \\
        & x^i_T \in x^i_g, \quad \forall i \in \{1, ..., k\} \quad && \label{eq:goal_satisfaction}
    \end{align}
\end{subequations}}
    
\noindent where $x^i_g$ is the goal state region for the $i^\textrm{th}$ robot. This optimization framework ensures that each robot computes a trajectory that adheres to safety constraints while avoiding deadlocks and reaching its goal, while minimizing deviation from its preferred trajectory and maintaining smooth control.

\section{BACKGROUND}
\label{Background}

We introduce key theoretical foundations necessary for understanding our framework. This section begins with an overview of Control Barrier Functions (CBFs), which serve as the backbone that enforces and ensures safety constraints in robotic navigation. Next, we review Depth-CBF \cite{desa2024_pointcloud}, an approach that leverages point cloud data for real-time collision avoidance in single agent environments. Establishing the background here provides necessary groundwork for our multi-agent point-cloud navigation approach and its integration with deadlock prevention strategies in later sections. 

\subsection{Control Barrier Functions}

Control Barrier Functions (CBFs) \cite{ames2016_cbf} are a mathematical framework used in control theory to ensure system safety while achieving desired control objectives.
A CBF is a scalar function, $h({x})$, defined over the state space of a system, where $\mathbf{x}$ represents the state of the system. The safe set, $\mathcal{C}$, is defined as the set of all states for which $h(\mathbf{x}) \geq 0$. Ensuring safety involves keeping the system state within $\mathcal{C}$ at all times. That is, safety is guaranteed as long as the CBF is always non-negative. This is achieved by designing a control input, $\mathbf{u}$, such that the following condition is satisfied:
{\small\begin{equation}
\dot{h}(x, u) \geq -\alpha(h(x)),
\label{eq:cbfcondition}
\end{equation}}
where $\dot{h}(x, u)$ is the derivative of $h(x)$, and $\alpha$ is an extended class $\mathcal{K}$ function, typically a linear or higher-order function that ensures the safety constraint is enforced with appropriate robustness. In practice, CBFs are integrated into a control optimization framework. Control-affine systems are utilized to facilitate usage of CBFs. In the continuous-time setting, a control-affine system is given by:

{\small\begin{equation}
\dot{x} = f(x) + g(x)u,
\end{equation}}
where $f(x)$ represents the system dynamics and $g(x)$ maps control inputs to state derivatives. The CBF condition becomes:
{\small\begin{equation}
\frac{\partial h}{\partial x} (f(x) + g(x)u) + \alpha(h(x)) \geq 0.
\end{equation}}

In our implementation, we utilize the fourth-order Runge-Kutta (RK4) method~\cite{githubGitHubHybridRoboticsdepth_cbf} to numerically integrate the continuous-time system over discrete time steps. Since our implementation operates in discrete time, we reformulate the CBF condition accordingly. The discrete-time counterpart of the CBF condition is given by: 




{\small\begin{equation}
    h(x_{t+1}) - h(x_t) + \alpha(h(x_t)) \Delta t \geq 0.
\end{equation}}


\subsection{Single Robot Navigation with Point Clouds}

The Depth-Based Control Barrier Function (Depth-CBF) framework~\cite{desa2024_pointcloud} ensures safety in real-time vision-based navigation by utilizing point cloud data as a direct representation of the robot's environment. 
The robot's position is denoted as $q \in \mathbb{R}^2$, and its environment is represented as a point cloud $P^i = \{p_i\}_{i=1}^N$, where each $p_i \in \mathbb{R}^2$ is a point obtained from sensors such as LiDAR or depth cameras. To define a safety margin, the control barrier function $h(q)$ is given by:

{\small\begin{equation}
    h(q) = \min_{p \in P} \{\|q - p\|^2 - \delta^2\}
    \label{eq:cbf_calc}
\end{equation}}

where $\delta > 0$ specifies the minimum safe distance. We utilize a default  $\delta$ value of 0.1. The condition $h(q) \geq 0$ ensures that the robot maintains the safety margin. Note that as per Equation~\ref{eq:cbfcondition}, to ensure safety over time, $h(q)$ must satisfy the following condition:
{\small\[
\dot{h}(q) + \alpha(h(q)) \geq 0.
\]}
The Depth-CBF is incorporated into a Quadratic Program (QP) to compute safe control inputs while remaining close to a nominal control input $k(q)$. The QP is formulated as:
{\small\begin{equation}
    u^* = \arg\min_u \frac{1}{2} \|u - k(q)\|^2
    \label{eq:QP}
\end{equation}
\[
\text{subject to } \nabla h(q)^T (f(q) + g(q)u) +\alpha(h(q))\geq 0,
\]}
where $f(q)$ and $g(q)$ describe the system dynamics in control-affine form:
$\dot{q} = f(q) + g(q)u$. This formulation ensures that the control input $u^*$ minimally deviates from the nominal input $k(q)$ while satisfying the safety constraints imposed by the Depth-CBF.
   \begin{figure*}[tb]
      \centering
      \includegraphics[scale=0.49]{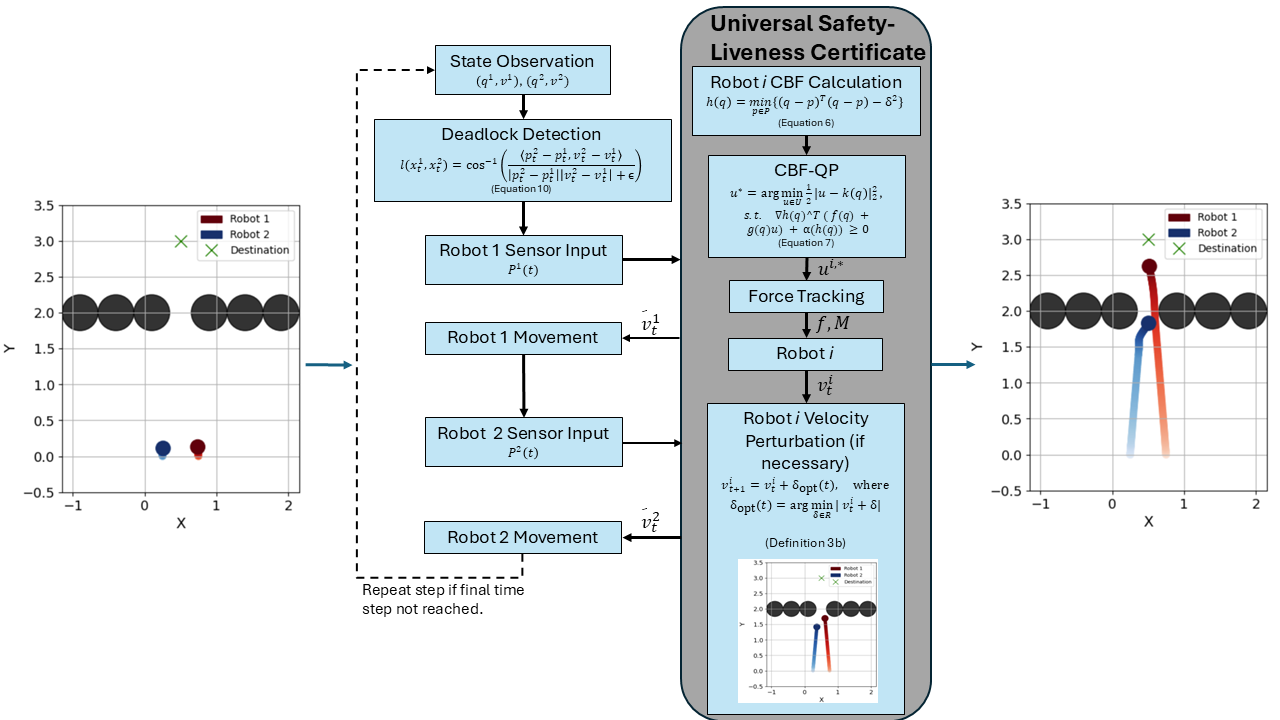}
      \caption{Technical flowchart illustrating our multi-agent navigation approach. Our universal safety-liveness certificate processes point cloud input to generate robot velocities that ensure deadlock-free navigation. Collectively, the blue elements represent one simulation step.}
      \label{technical_flow}
      \vspace{-10pt}

   \end{figure*}

\section{\framework{}: UNIVERSAL SAFETY AND LIVENESS CERTIFICATES}
\label{Approach}
In this section, we introduce \framework{} for safe and deadlock-free multi-robot navigation using real-time point cloud perception in SMGs. To encourage liveness, we present a deadlock prevention strategy, which quantifies trajectory symmetry and applies minimal velocity adjustments to maintain progress toward goals. We outline the computational framework that integrates safety constraints and deadlock prevention into a single, decentralized control policy.

\subsection{Safety} To ensure safe and collision-free navigation, we employ CBFs to define safety constraints that prevent robots from colliding with static and dynamic obstacles. These CBF constraints are derived solely from high-dimensional point cloud data. At each step, each robot perceives its environment using point cloud data $P^i$, which allows it to dynamically compute safety constraints at every time step. With the point cloud input, we use Equation~\ref{eq:cbf_calc} to our Control Barrier Function. Next, we compute a control input utilizing a modified CBF-QP controller. Our controller is formulated using Equation~\ref{eq:QP}, which is valuable as it enables control utilizing only a single optimization constraint. The calculated control input $u^{i,*}$ is provided to a force-tracking controller, which generates a thrust $f \in \mathbb{R}$ and a moment $M \in \mathbb{R}^3$. These control inputs enable the quadrotor to track the computed safe $u^{i,*}$. The force $f$ and moment $M$ are then utilized to determine the quadrotor's movement, as according to its translational and rotational dynamics, as follows:

{\small\begin{equation}
    \dot{q} = v, \quad m\dot{v} = f R e_3 - mg e_3,
    \label{eq:quadrotor_dynamics1}
\end{equation}}

{\small\begin{equation}
    \dot{R} = R\hat{\omega}, \quad I \dot{\omega} + \hat{\omega} I \omega = M,
    \label{eq:quadrotor_dynamics2}
\end{equation}}

where $\hat{\cdot}$ maps a 3D vector to a skew-symmetric matrix in ${so}(3)$ such that $\hat{x} y = x \times y$ for all $x, y \in \mathbb{R}^3$, $m$ is the quadrotor mass, $R$ is the rotational matrix, $f$ is the total rotor thrust, $e_3$ is the Euclidean basis vector $[0,0,1]^\top$, $I$ is the inertia tensor, $\omega$ is the angular velocity, $M$ is the total moment applied, and $g$ is the acceleration due to gravity.



\subsection{Liveness} Our two-part approach to ensuring deadlock-free navigation consists of deadlock detection and resolution. Deadlock resolution is only triggered if a potential deadlock is flagged.

\paragraph{Deadlock Detection} At every time step, before any movement occurs, we check for a potential deadlock using a liveness function, which captures the symmetry of the environment and is defined as:

   {\small\begin{equation} 
        \ell\big(x^i_t, x^j_t\big) = \cos^{-1} \left( \frac{\langle p^j_t - p^i_t, v^j_t - v^i_t \rangle}{| p^j_t - p^i_t | | v^j_t - v^i_t | + \epsilon} \right), 
    \label{eq:liveness_function}
    \end{equation}}

    \noindent where $\langle a, b \rangle$ denotes the dot product between vectors $a$ and $b$, $| \cdot |$ denotes the Euclidean norm, $\ell \in [0, \pi]$, and $\epsilon > 0$ ensures that the denominator is positive. The liveness function gives the degree of symmetry by measuring the angle between the relative displacement and relative velocity of two robots. In cases of near-perfect symmetry, these vectors align, making the angle approach zero. If $\ell_{ij}(x^i, x^j)(t) \leq \ell_{thresh}$, where $\ell_{thresh}$ is a liveness threshold, a deadlock is detected. We utilize $\ell_{thresh} = 0.3$. To explain why this value is used for the liveness threshold, we introduce scaling factor $\zeta$, which represents the ratio between two agents' velocities. The minimum value of $\zeta$ that determines whether one robot can clear the constrained region of a SMG before the other arrives is presented in the following theorem:

    \begin{theorem}
 Consider a symmetric SMG, such as the one in Figure \ref{smg}, with two robots length $l$ that are at distance $d$ from the point of collision $\mathcal{Q}$. For the robots to reach $\mathcal{Q}$ without colliding, one of the robots must slow down (or speed up) by a factor of $\zeta \geq 2$ in the limit as $(d, \epsilon) \to (l, 0)$.
    \label{thm: factor_of_2}
\end{theorem}


\begin{proof}

In a symmetric SMG consisting of two robots, we assume that the speeds of the robots are nearly identical. In practice, two speeds will likely differ by some small amount, which we denote as $\epsilon$ giving $0 < \lvert v^1_t - v^2_t\rvert \leq \epsilon$. Suppose, without loss of generality, that robot $1$ is slightly faster than robot $2$ by $\epsilon$. Thus, $v^1_t - v^2_t = \epsilon \implies v^2_t = v^1_t -\epsilon$. In order for robot $1$ to clear the distance $d$ in the time that robot $2$ reaches the $\mathcal{Q}$, robot $1$ must increase its speed by a factor of $\zeta$. Therefore $v^1_{t+1} \gets \zeta v^1_{t}$ while $v^2_{t+1} = v^1_t -\epsilon$. We assume that speed changes are instantaneous as the agents enter the SMG.

Now time taken by robot 2 to reach Q is $t_2 = \frac{d}{v^2_{t+1}} = \frac{d}{v^1_t -\epsilon}$. And the time taken by robot 1 to \textit{clear} Q (that is, not just reach Q, but that the entire length $l$ should clear Q) is $t_1 = \frac{d + l}{v^1_{t+1}} = \frac{d + l}{\zeta v^1_t}$. Note that $t_2 \geq t_1$. Thus,

\begin{equation*}
    \frac{d}{v^1_t -\epsilon} \geq \frac{d + l}{\zeta v^1_t} 
\end{equation*}

Rearranging terms gives us Equation (13),

\begin{equation}
\zeta \geq \left(1 + \frac{l}{d}\right)\left(1 - \frac{\epsilon}{v^1_t}\right)     
\end{equation}

In the limit, $\zeta \to 2$ as $(d, \epsilon) \to (l, 0)$.
\end{proof}

From $\zeta = 2$, we find our relevant $\ell_{thresh}$ by noting that there is a relationship between the two values~\cite{chandra2024_socialgames}. The deadlock threshold corresponds to:

\begin{equation}
    \ell_j = \frac{\pi}{4} - \tan^{-1}\frac{1}{\zeta}
    \label{eq: l_thresh}
\end{equation}

Thus, from $\zeta = 2$, our corresponding $\ell_{thresh}$ is 0.3. We empirically validate this threshold through the results of our experiments.






\paragraph{Deadlock Prevention}
   
   If deadlock was detected, our goal is for robots to decentrally avoid the deadlock by perturbing its speed in a minimally invasive manner. The idea is to design an approach that only changes the speed of the robot and not its planned trajectory or actual positions. In order to inform our velocity-based deadlock avoidance approach, we introduce the concept of \textit{liveness sets}, inspired by the concept of safety sets in CBFs. Liveness sets define a space fo joint speeds for all agents such that, if agents travel at those speeds, they will remain deadlock-free. 

    \begin{definition}
At any time $t$, given a configuration of $k$ robots, $x^i_t \in \mathcal{X}$ for $i\in [1,k]$, a \textbf{liveness set} is defined as a union of convex sets, $\mathscr{C}_\ell(t) \subseteq \mathbb{R}^k$ of joint speed $v_t = \left[v^1_t, v^2_t, \ldots, v^k_t\right]^\top$ such that $v^i_t \geq \zeta v^j_t$ for all distinct pairs $i,j$, $ \zeta \geq 2$. 

    \label{def: liveness_sets}
\end{definition}

As we showed in Theorem \ref{thm: factor_of_2} that a $\zeta \geq 2$ is required to guarantee that robots can navigate through a symmetric SMG without deadlocks, liveness sets guarantee that if $v_t \in \mathscr{C}_\ell(t)$, any pair of robots in our configuration will have feasible control inputs that guarantee forward motion. But if $v_t \notin \mathscr{C}_\ell(t)$, then we must perform a minimally invasive velocity perturbation. We define a minimally invasive perturbation as:
    \begin{definition}
        Minimally Invasive Deadlock Resolution: A deadlock-resolving strategy prescribed for robot $i$ at time $t$ with current heading angle $\theta_t^i$ is minimally invasive if:
    \begin{enumerate}
        {\small\item $\Delta \theta^i_t = \theta^i_{t+1} - \theta^i_t = 0$ (The agent does not deviate from the preferred trajectory).
        \item $v_{t+1}^i = v_t^i +\delta_{\text{opt}}(t)$, where $\delta_{\text{opt}}(t) = \arg\min_{\delta \in \mathbb{R}}\left\Vert v_t^i + \delta \right\rVert,$ such that $v_{t+1}^i \in \mathscr{C}_\ell(t)$ (i.e., robot with speed $v_{t+1}^i$ prevents or resolves a deadlock).  }
    \end{enumerate}
    \end{definition}


 \noindent $\delta_{\text{opt}}(t)$ can be found by solving the following optimization problem: 
 

        {\small\begin{subequations}
                \begin{align}
        \delta_{\text{opt}}(t)=&\underset{\delta\in\mathbb{R}}{\arg\min}\;\|v^i_t+\delta\|,
        \label{eqn:minimally_invasive_perturbation}\\
        &v^i_{t+1}=v_t^i+\delta\\
        &u^i_t\in\mathscr{U}^i,\;\;u^i_{t+1}\in\mathscr{U}^i\\
        &\left(x^i_{t+1},u^i_{t+1}\right)\notin\mathcal{D}^i(t+1) 
        \label{eqn:minimally_invasive_optimization}
        \end{align}
                \end{subequations}}

    where deadlock set $\mathcal{D}^i$ is defined as follows: 
    {\small\begin{align}
        \mathcal{D}^i(t)=\left\{\left(x^i_t,u^i_t\right):x^i_t\notin\mathcal{X}_g,\;\;u^i_t=0 \  \textnormal{for some threshold} > 0\right\}
        \label{eqn:deadlock_set}
    \end{align}}
    If each $v^i_t\in \mathscr{C}_\ell(t)$, then there is no deadlock. If, however, $v_t \notin \mathscr{C}_\ell(t)$, then robot $i$ will adjust $v^i_t$ such that $v_t$ is projected onto the nearest point in $\mathscr{C}_\ell(t)$. We establish that our approach ensures both safety and liveness guarantees. The use of Control Barrier Functions (CBFs) inherently enforces safety by ensuring that the system remains within a certified safe set. To guarantee liveness, we demonstrate that a feasible velocity perturbation-based solution always exists and is unique under reasonable assumptions.



\begin{theorem} A solution to the optimization problem given in Equation~\ref{eqn:minimally_invasive_perturbation} always exists and is unique if \( v^i_t \neq v^j_t \) for any robots $i$ and $j$.
\end{theorem}

\begin{proof}
    The set of feasible velocities, given by \( \mathscr{C}_\ell(t) \), is the complement of an open complex polytypic set $P$, whose boundary \( \partial P \) is a subset of \( \mathscr{C}_\ell(t) \). Since $P$ is open and convex, the minimum Euclidean distance projection from any interior point of $P$ to \( \partial P \) always exists, as \( \partial P \) is a compact set with a finite number of edges and the distance function is continuous. Thus a solution always exists.

    Uniqueness follows from the fact that the projection of a point onto the boundary of a convex polytype is unique unless the point is equidistant from multiple boundary points. This case only occurs if the robot $i$ and $j$ have identical velocities, which contradicts our assumption of \( v^i_t \neq v^j_t \). Thus, our solution is unique.    
\end{proof}
The approach outlined in this section make up the Liveness aspect of our Universal Safety-Liveness Certificate, as seen in Figure~\ref{technical_flow}.

\subsection{Overall Algorithm}
Our sequential algorithm is formulated as follows and  illustrated in Figure~\ref{technical_flow}. First, each robot observes its state. The liveness function (Equation~\ref{eq:liveness_function}) is then computed and evaluated against the threshold to determine whether the robots are at risk of a deadlock. Next, Robot 1 perceives its point cloud, computes its CBF constraints, and computes its control input and corresponding force, moment, and desired velocity. If a potential deadlock was detected earlier, the deadlock resolution procedure is executed as per Equation~\ref{eqn:minimally_invasive_perturbation}. Robot 1 then updates its state. 
Next, Robot 2 follows the same process: perceiving its point cloud, computing its CBF constraints, and determining its control input. If a potential deadlock was previously detected, the deadlock resolution procedure is applied. This process repeats iteratively until both robots have reached their destinations, a collision occurs, or maximum time step $T$ is reached. 

\section{EXPERIMENTS \& RESULTS}
\label{Experiments}

In this section, we evaluate the performance of our proposed framework, \framework{}, in multi-robot navigation scenarios. The goal of our experiments is to answer the following key questions: $(i)$ does \framework{} successfully prevent collisions in constrained multi-agent environments using high-dimensional inputs? $(ii)$ does \framework{} avoid deadlocks and ensure continuous motion for all agents? and $(iii)$ how does \framework{} compare to existing methods in terms of successful navigation and smoothness? To answer these questions, we conduct experiments in a Python-based simulator~\cite{githubGitHubHybridRoboticsdepth_cbf} with quadrotor robots. We first discuss the simulation setup, including key parameters and environmental conditions. We then describe the baselines and comparison methods used for evaluation. Finally, we present and analyze the results, demonstrating the impact of \framework{} on multi-agent safety, liveness, and navigation efficiency. 

\subsection{Simulation Design and Experiment Setup}
We evaluate \framework{} in a Python-based simulation environment~\cite{githubGitHubHybridRoboticsdepth_cbf}, simulating two robotic agents ($k=2$) in constrained environments. We assess navigation in a doorway scenario and an intersection scenario, both requiring robot coordination to avoid collision. Each experiment is run until both robots reach their respective goals, a collision occurs, or movement ceases due to a deadlock lasting the duration of the simulation. Each robot updates its motion every 0.02 seconds, meaning a full control update (both robots moving once) occurs every 0.04 seconds. The maximum number of time steps is set at 400 ($T=400$), capping the total runtime at 16 seconds. If neither robot reaches its goal within this limit, the run is classified as a deadlock. Collision detection is performed by modeling robots as spheres with a 0.1 m radius, where a collision is registered if the center-to-center distance between two robots is less than 0.2 m. Static obstacles, represented as spheres with a 0.25 m radius, are placed at predefined positions. A robot is considered to have reached its destination once its center is within 0.1 m of its goal position, at which point it is removed from the simulation. The quadrotor robots begin from rest, and no explicit constraints are imposed on their maximum velocity, acceleration, deceleration, or angular velocity. The quadrotor's dynamics follow Equations \ref{eq:quadrotor_dynamics1} and \ref{eq:quadrotor_dynamics2}.


We evaluated these approaches in two SMG scenarios:
\begin{itemize}
\item \textbf{Doorway SMG:} Two robots navigate toward a common destination through a narrow doorway. The two robots start equidistant from a narrow 0.3m-wide passage.
\item \textbf{Intersection SMG:} Two robots approach a four-way intersection at right angles. The two robots start equidistant from the intersection. 
\end{itemize}

\subsection{Baselines and Metrics}
To comprehensively evaluate \framework{}, we conduct comparisons against the following baselines, as well as ablation studies:
\begin{enumerate}
\item \textbf{Depth-CBF (Single-Agent)} \cite{desa2024_pointcloud}: We utilize Depth-CBF to control a single robot while treating the second robot as a dynamic obstacle with predefined positions. In contrast, our approach enables two robots to navigate independently in a fully decentralized manner. 
\item \textbf{\framework{} without Liveness (LoL)}: This is an ablation baseline where we eliminate our deadlock prevention algorithm. 
\item \textbf{MPC-CBF}~\cite{zeng2021mpc}: MPC-CBF aligns with our approach of leveraging CBFs for safety. It combines Model Predictive Control (MPC) with Control Barrier Functions (CBFs) for collision avoidance. Unlike \framework{}, MPC-CBF explicitly formulates an optimization problem over a finite horizon. 
\item \textbf{MPNet}~\cite{mpnet2020}: MPNet is a state of the art learning-based motion planner which uses neural networks to generate paths directly from raw point cloud data. Its reliance on point cloud data aligns with our approach, making it a relevant baseline for comparison. We expand MPNet to a multi-robot setup, with the two robots trained on the same neural net policy.
		
\item \textbf{ORCA}~\cite{orca2020}: ORCA is a widely used collision avoidance method that utilizes safe velocity sets to generate collision-free motion for multiple agents. We adapt ORCA to our simulation by implementing its collision avoidance strategy while keeping the overall approach to generating CBFs unchanged. ORCA's reactive collision avoidance approach is analogous to our liveness algorithm, as both dynamically adjust agent behavior to ensure safe navigation in multi-robot environments.
\end{enumerate}

For each scenario, we perform 50 independent runs and measure key performance metrics, which were recorded for all scenarios and summarized in Table~\ref{table_combined}. 

We track the number of failed runs due to robot collisions (Collisions), the number of failed runs due to robot deadlocks (Deadlocks), and the percentage of successful runs where both robots reached their destinations without collision or deadlock (Success \%). We measure the time required for Robot 1 and Robot 2 to reach their respective destinations (Time R1, R2), with N/A indicating failure to reach the goal. Additionally, we track each robot's agility using average velocity (Vel R1, R2) while navigating through the constrained region (e.g., doorway or intersection), and also measure the magnitude of velocity changes through the constrained region ($|\Delta V|$). This metric reflects the smoothness of motion, with N/A assigned if a collision or deadlock occurred before reaching the constrained region. For our velocity metrics comparisons (velocity and ($|\Delta V|$), we focus specifically on Robot 2, as it is the agent that slows more significantly due to our velocity perturbation approach. Thus, understanding Robot 2's agility and smoothness of its deceleration provides the most insight into effective strategies for deadlock prevention.

    \begin{figure}[t]
        \centering
        \begin{subfigure}{0.325\columnwidth}
            \includegraphics[width=\textwidth,height=0.115\textheight]{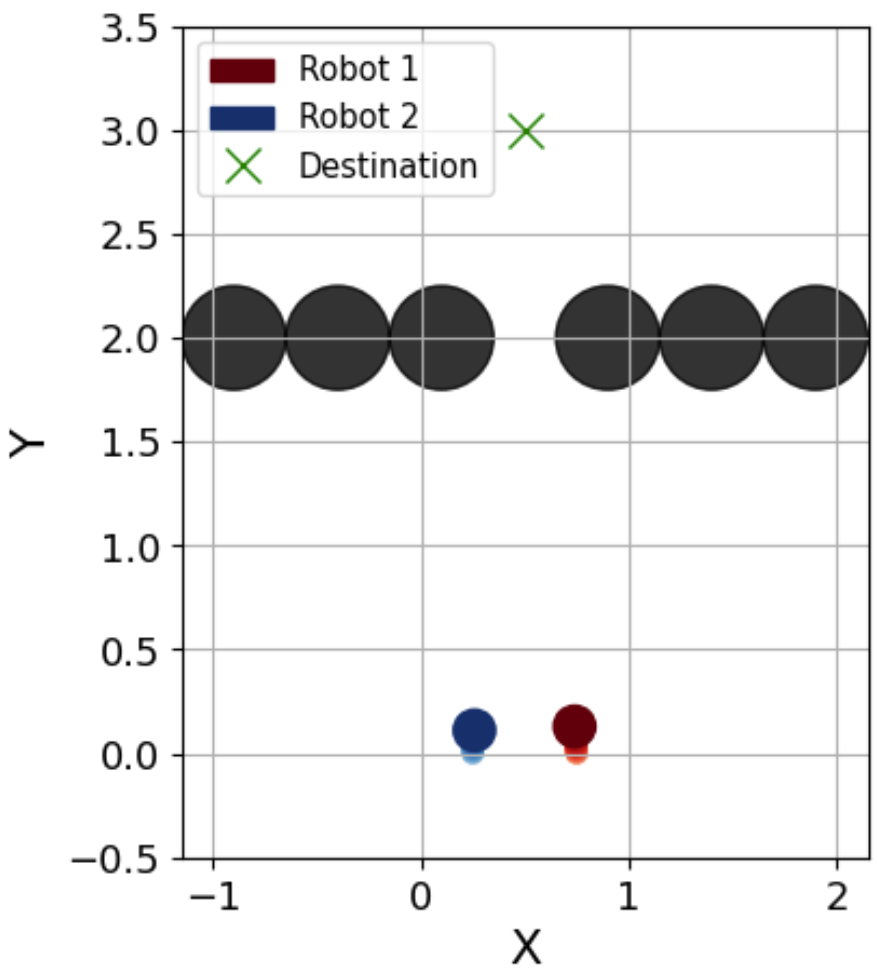}
            \caption{Robots moving towards the opening.}
            \label{fig:1a}
        \end{subfigure}
        \begin{subfigure}{0.325\columnwidth}
            \includegraphics[width=\textwidth,height=0.115\textheight]{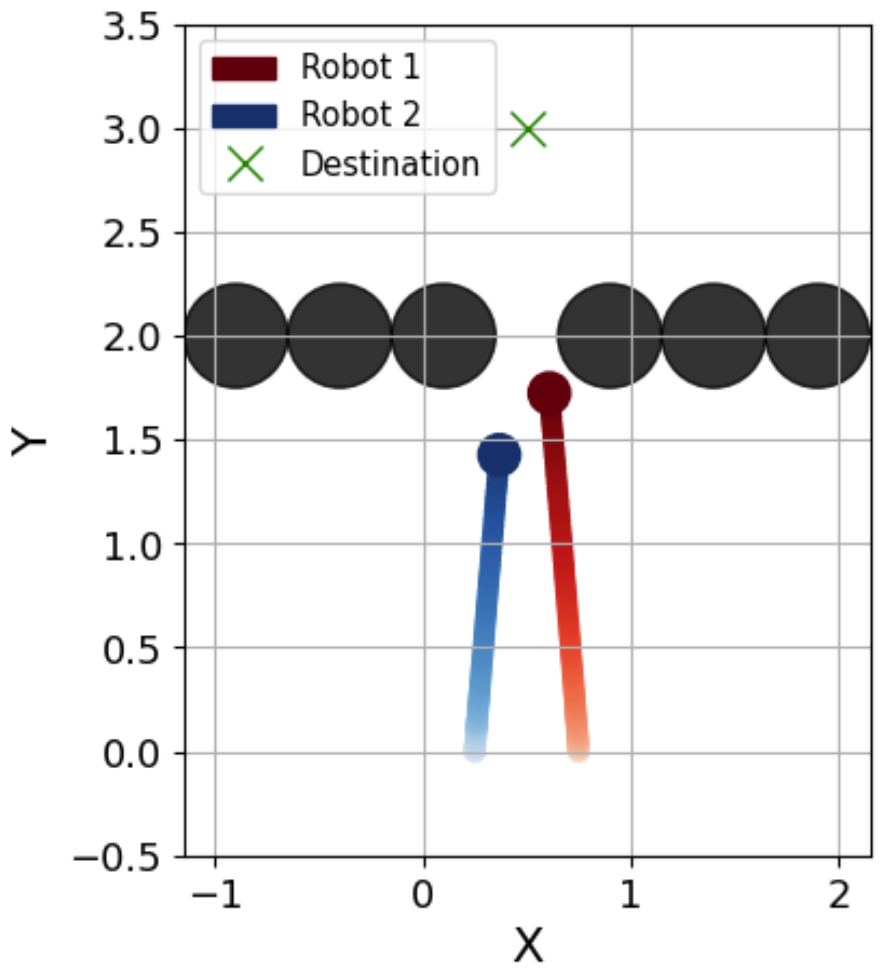}
            \caption{Robot 1 proceeds first.}
            \label{fig:1b}
        \end{subfigure}
        \begin{subfigure}{0.325\columnwidth}
            \includegraphics[width=\textwidth,height=0.115\textheight]{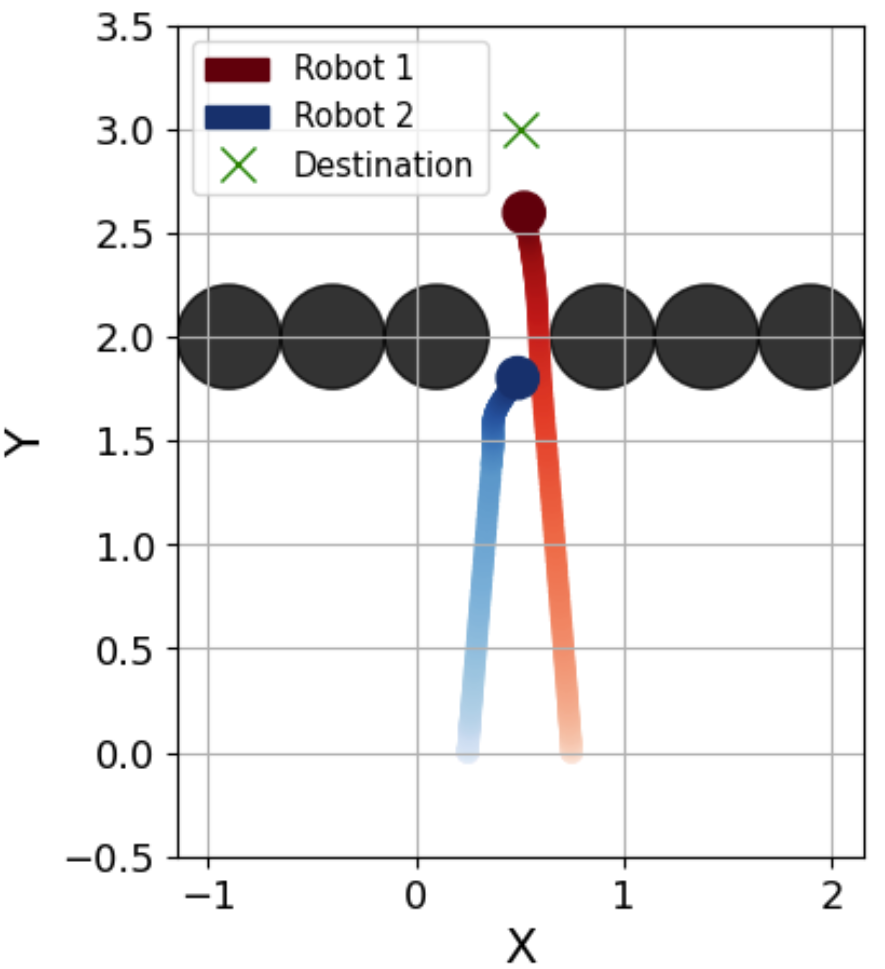}
            \caption{Robot 2 follows 1 through the opening.}
            \label{fig:1c}
        \end{subfigure}
        \caption{Doorway: Multi-Agent Robot Trajectories with Liveness}
        \label{liveness}
        \vspace{-10pt}
    \end{figure}

\subsection{Results and Observations}

\begin{table*}[t]
\captionsetup{position=bottom}
\begin{center}
\resizebox{\textwidth}{!}{%
\begin{tabular}{lccccccccc}
\toprule[1.2pt]
\multicolumn{9}{c}{\textit{Doorway Scenario}} \\
\midrule
Method & Collisions $\downarrow$ & Deadlocks $\downarrow$ & Success(\%) $\uparrow$& Time R1 (s) $\downarrow$& Time R2 (s) $\downarrow$& Vel R1 (u/s) $\uparrow$& Vel R2 (u/s) $\uparrow$& $|\Delta V|$ R1 (u/s) $\downarrow$&$|\Delta V|$ R2 (u/s) $\downarrow$\\
\midrule
ORCA \cite{orca2020}& 50 & 0 & 0 & N/A & N/A & \textbf{1.4178} & \textbf{1.419} & \textbf{1.971e-02} & 1.975e-02 \\
MPC-CBF \cite{zeng2021mpc}& 50 & 0 & 0 & N/A & N/A & N/A & N/A & N/A & N/A \\
MPNet \cite{mpnet2020}& 50 & 0 & 0 & N/A & N/A & N/A & N/A & N/A & N/A \\
Single & 50 & 0 & 0 & N/A & N/A & 1.248 & N/A & 2.294e-02 & N/A \\
\framework{} w/o Liveness & 50 & 0 & 0 & N/A & N/A & 1.1209 & 1.0953 & 2.356e-02 & 2.279e-02 \\
\framework{} & \textbf{0} & \textbf{0} & \textbf{100} & \textbf{9.12} & \textbf{11.68} & 1.0332 & 0.3143 & 2.446e-02 & \textbf{1.424e-02} \\
\toprule[1.2pt]
\multicolumn{9}{c}{\textit{Intersection Scenario}} \\
\midrule
ORCA\cite{orca2020} & 50 & 0 & 0 & N/A & N/A & \textbf{3.8008} & \textbf{3.2846} & 1.255e-01 & 5.035e-01 \\
MPC-CBF \cite{zeng2021mpc}& 0 & 50 & 0 & N/A & N/A & N/A & N/A & N/A & N/A \\
MPNet \cite{mpnet2020} & 50 & 0 & 0 & N/A & N/A & N/A & N/A & N/A & N/A \\
Single & 50 & 0 & 0 & N/A & N/A & 2.5753 & N/A & 1.207 e-01 & N/A \\
\framework{} w/o Liveness & 50 & 0 & 0 & N/A & N/A & 0.9032 & 0.7801 & \textbf{8.2195e-03} & \textbf{1.1649e-02} \\
\framework{} & \textbf{0} & \textbf{0} & \textbf{100} & \textbf{6.56} & \textbf{9.84} & 2.6699 & 2.7831 & 2.808e-01 & 9.7697e-02 \\
\bottomrule[1.2pt]
\end{tabular}%
}
\end{center}
\caption{Experiment Results for Doorway and Intersection Scenarios, averaged over 50 runs. The best results for each category are bolded, with $\uparrow$ indicating a higher value is preferable, and $\downarrow$ indicating a lower value is preferable. R1 and R2 denote Robot 1 and Robot 2, respectively. Time denotes time required to reach the goal, Vel represents the velocity through the constrained region (i.e., doorway or intersection), and $|\Delta V|$ quantifies velocity smoothness, measured as the magnitude of velocity change within the constrained region.} 
\label{table_combined}
\vspace{-10pt}
\end{table*}

    \begin{figure}[t]
        \centering
        \begin{subfigure}[b]{0.49\linewidth}
            \includegraphics[width=0.85\textwidth]{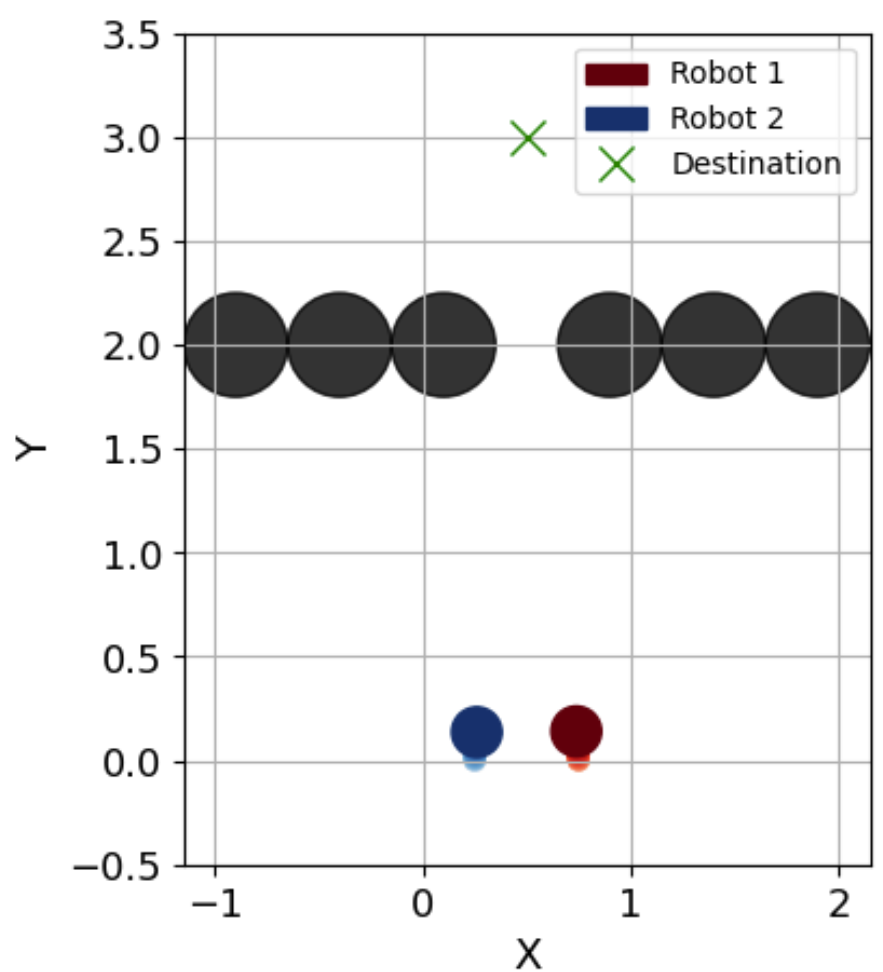}
            \caption{Robots moving towards the opening.}
            \label{fig:1a}
        \end{subfigure}
        \hfill
        \begin{subfigure}[b]{0.49\linewidth}
            \includegraphics[width=0.85\textwidth]{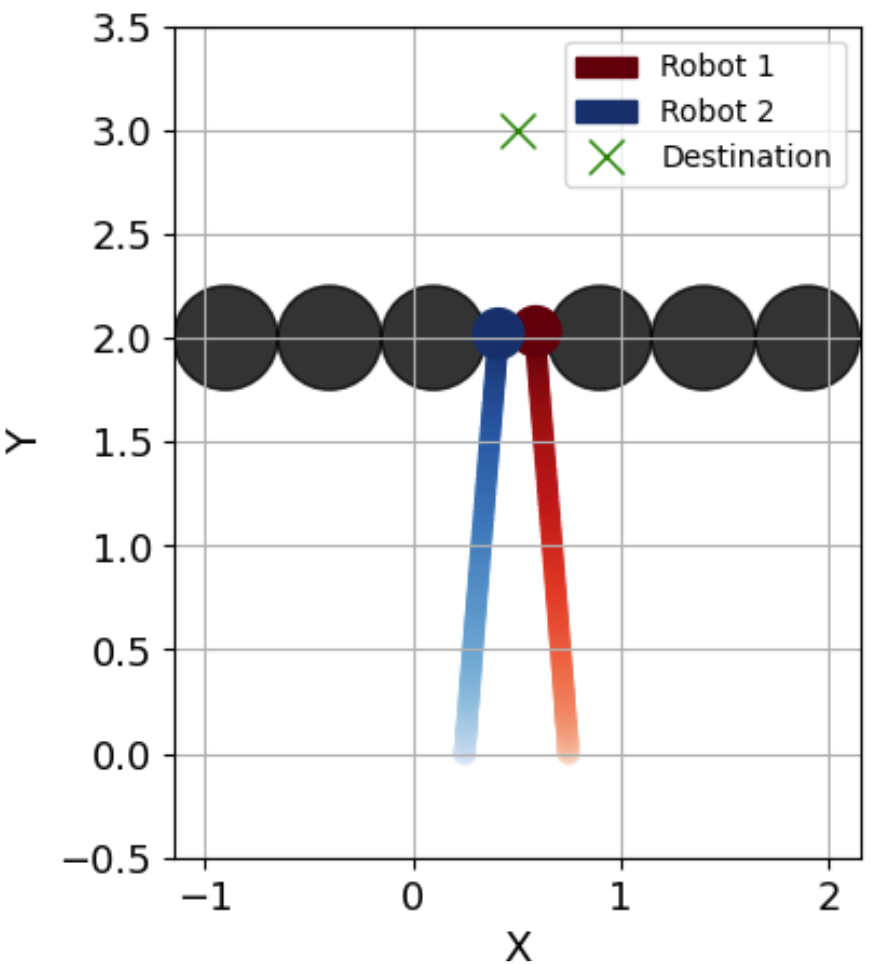}
            \caption{Robots attempt to enter doorway simultaneously, causing a collision.}
            \label{fig:1b}
        \end{subfigure}
        \hfill
    
        \caption{Doorway: Multi-Agent Robot Trajectories, No Liveness}
        \label{no_liveness}
    \end{figure}

   \begin{figure}[t]
        \centering
        \begin{subfigure}[b]{0.49\linewidth}
            \includegraphics[width=0.85\linewidth]{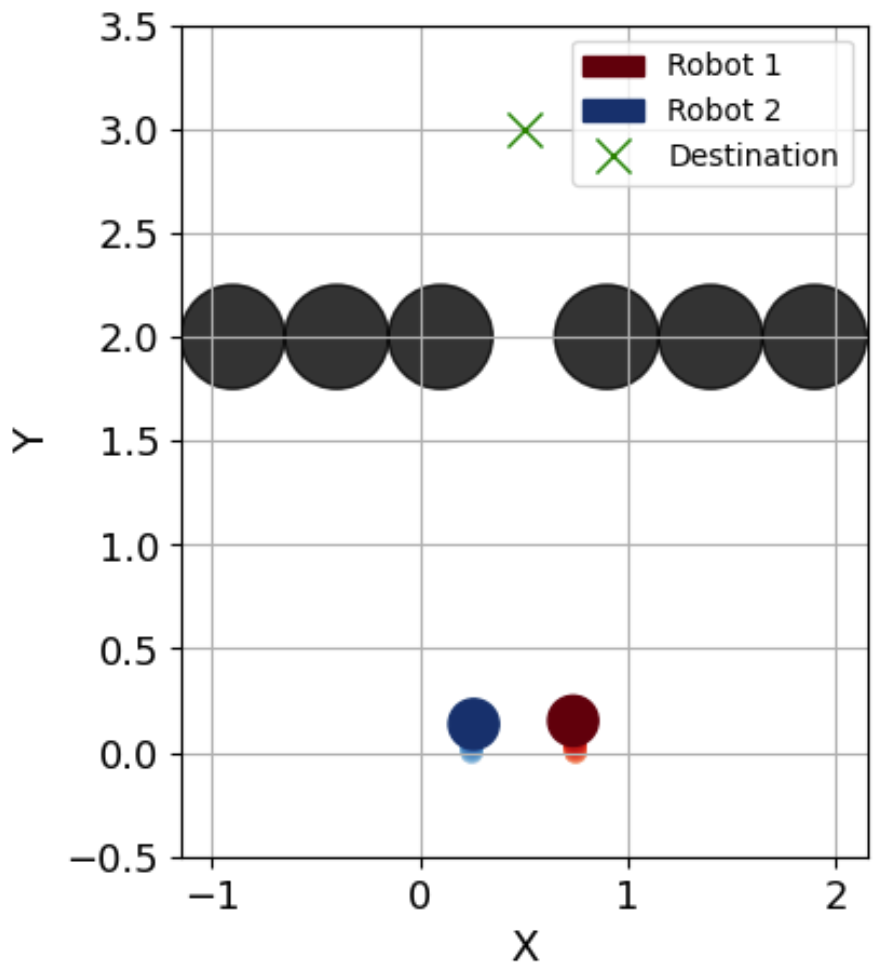}
            \caption{Robots moving towards the opening.}
            \label{fig:1a}
        \end{subfigure}
        \begin{subfigure}[b]{0.49\linewidth}
            \includegraphics[width=0.85\linewidth]{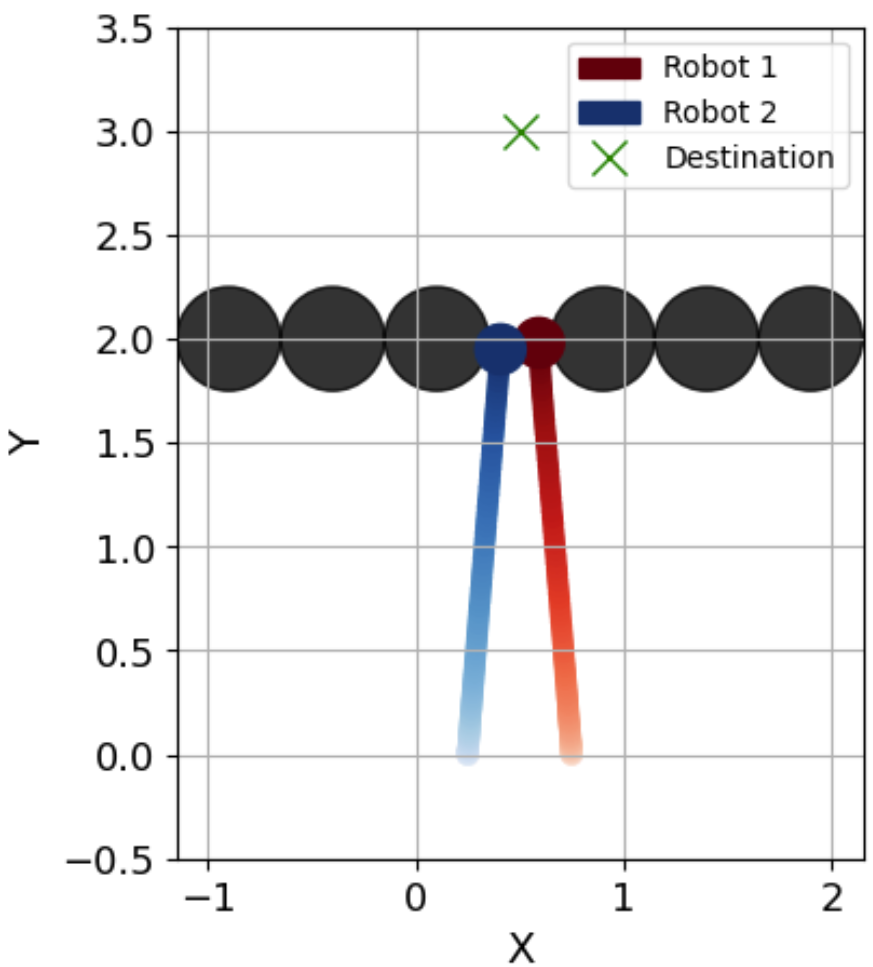}
            \caption{Robot 1 collides with dynamic obstacle.}
            \label{fig:1b}
        \end{subfigure}    
        \caption{Doorway: Single Agent Robot Trajectory, with Dynamic Obstacle Representing Robot 2}
        \label{hybrid}
    \end{figure}

\paragraph{Doorway SMG}

Figure~\ref{liveness} illustrates the robot trajectories under \framework{}. Potential deadlocks are detected early, prompting Robot 2 to reduce its speed and allow Robot 1 to pass through the doorway first. Once Robot 1 clears the doorway, Robot 2 takes its turn through the doorway.

In contrast, Figure~\ref{no_liveness} shows the results of LoL. Both robots attempt to enter the doorway simultaneously, resulting in a collision and a failed run. Similarly, when we run Depth-CBF on a single agent, with another pre-programmed dynamic obstacle representing a second robot, we also get a collision in the doorway, as seen in Figure~\ref{hybrid}.





\begin{figure}[t]
    \centering
    \begin{subfigure}{0.49\columnwidth}
        \centering
        \includegraphics[width=\linewidth]{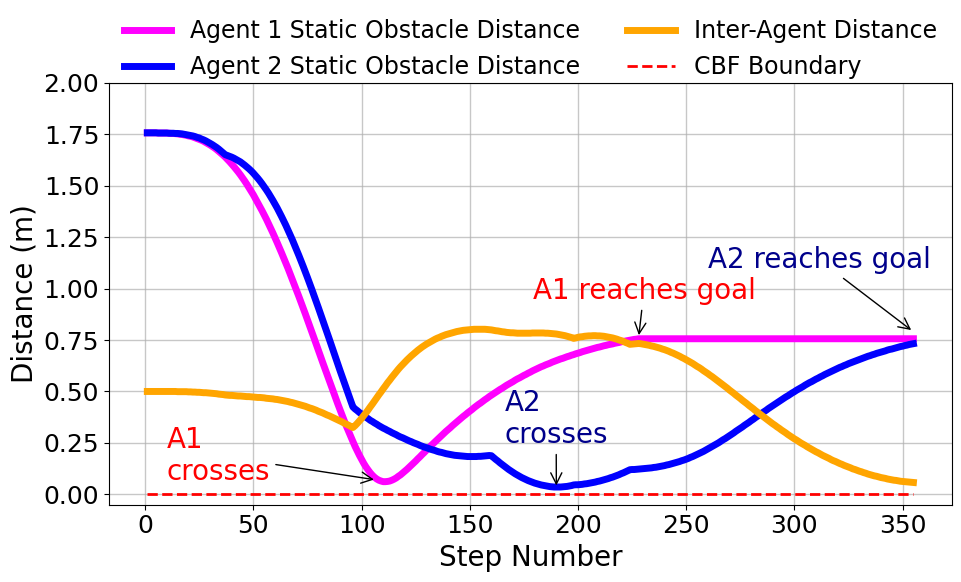}
        \caption{Distance CBF, Doorway}
        \label{door_distance_cbf}
    \end{subfigure}
    \hfill
    \begin{subfigure}{0.49\columnwidth}
        \centering
        \includegraphics[width=\linewidth]{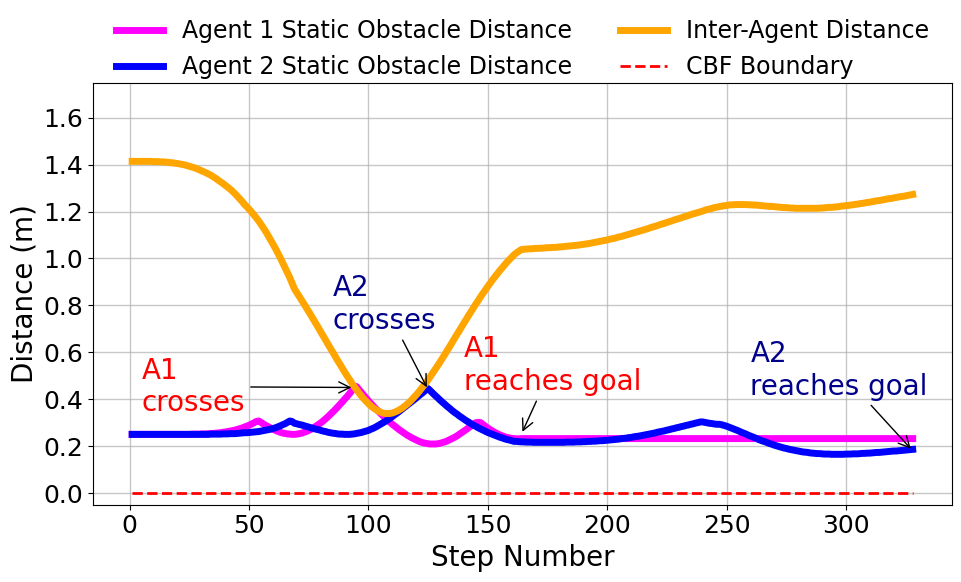}
        \caption{Distance CBF, Intersection}
        \label{int_distance_cbf}
    \end{subfigure}
    \caption{Distances between each agent and the closest static obstacle, as well as the inter-agent distance, plotted over time. The CBF boundary is included to indicate the safety threshold for maintaining safe navigation.}
    \label{combined_cbf}
\end{figure}

Figure~\ref{door_distance_cbf} illustrates how \framework{} maintains safe distances throughout the simulation. It shows each agent's distance to the closest static obstacle and inter-agent distances at every step. The plot highlights Agent 1 reaching the goal first, demonstrating effective movement coordination between robots through the doorway. Figure~\ref{door_liveness_cbf} highlights the liveness values calculated at each step. When the liveness value falls below the threshold of 0.3, the deadlock avoidance algorithm is triggered, ensuring proactive adjustments to prevent potential conflicts. This behavior occurs multiple times during the simulation, particularly in early steps where potential deadlocks are most likely to arise.

\begin{figure}[t]
    \centering
    \begin{subfigure}{0.49\columnwidth}
        \centering
        \includegraphics[width=\linewidth]{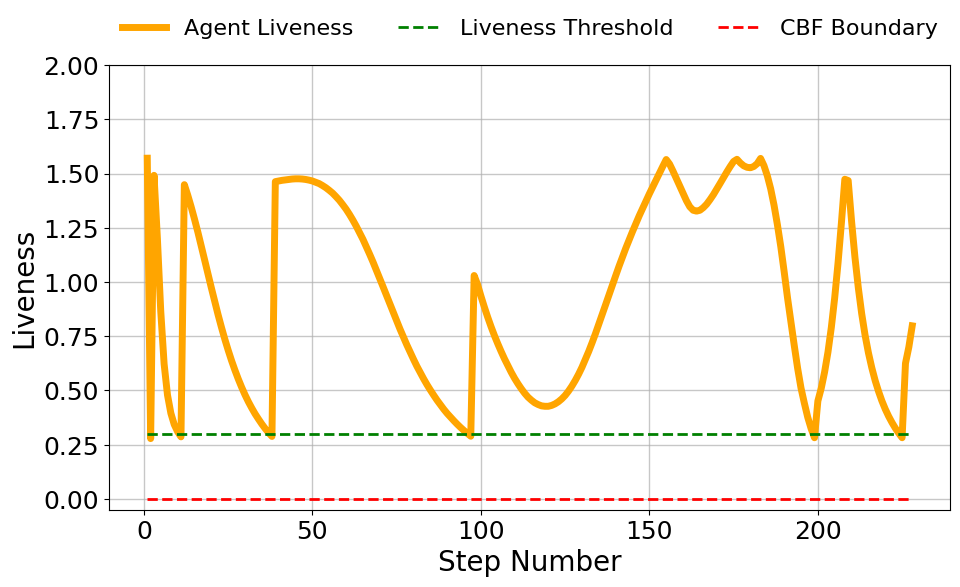}
        \caption{Liveness CBF, Doorway}
        \label{door_liveness_cbf}
    \end{subfigure}
    \hfill
    \begin{subfigure}{0.49\columnwidth}
        \centering
        \includegraphics[width=\linewidth]{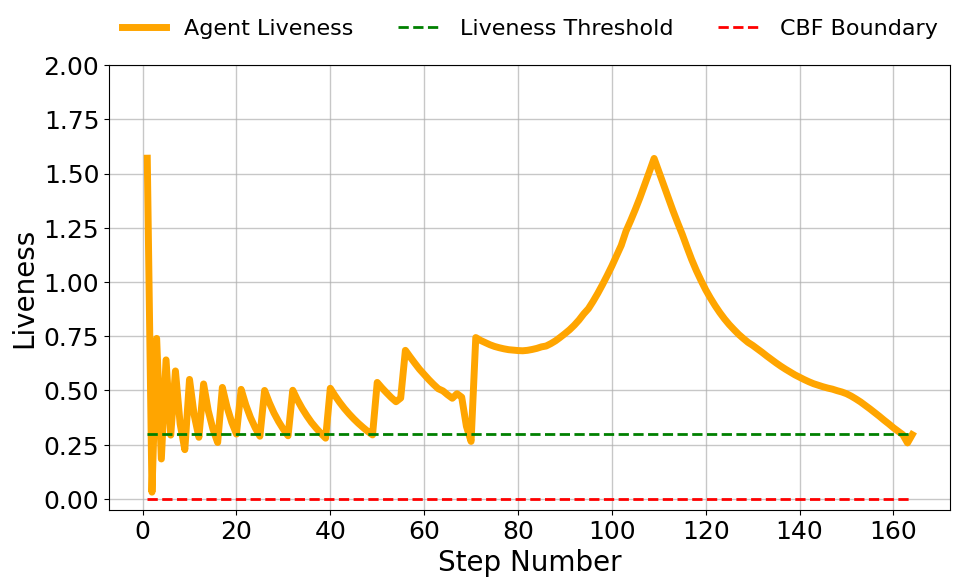}
        \caption{Liveness CBF, Intersection}
        \label{int_liveness_cbf}
    \end{subfigure}
    \caption{Liveness values computed at each time step, with liveness threshold and CBF boundary overlaid. Velocity perturbation is performed when liveness value falls below the liveness threshold.}
    \label{liveness_cbf}
\end{figure}

Figure~\ref{velocity_doorway} illustrates the velocity profiles of both robots before and after the doorway. The seamless coordination highlights \framework{}’s effectiveness in resolving potential conflicts in real time.



\begin{figure}[t]
    \centering
    \begin{subfigure}{0.49\columnwidth}
        \centering
        \includegraphics[width=\linewidth]{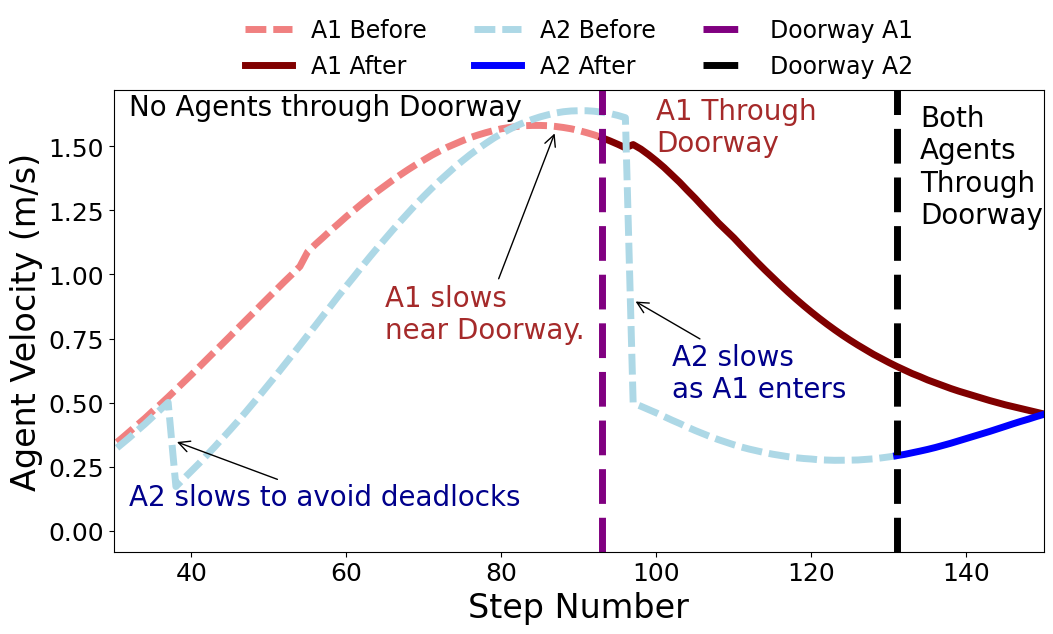}
        \caption{Agent velocities before and after doorway}
        \label{velocity_doorway}
    \end{subfigure}
    \hfill
    \begin{subfigure}{0.49\columnwidth}
        \centering
        \includegraphics[width=\linewidth]{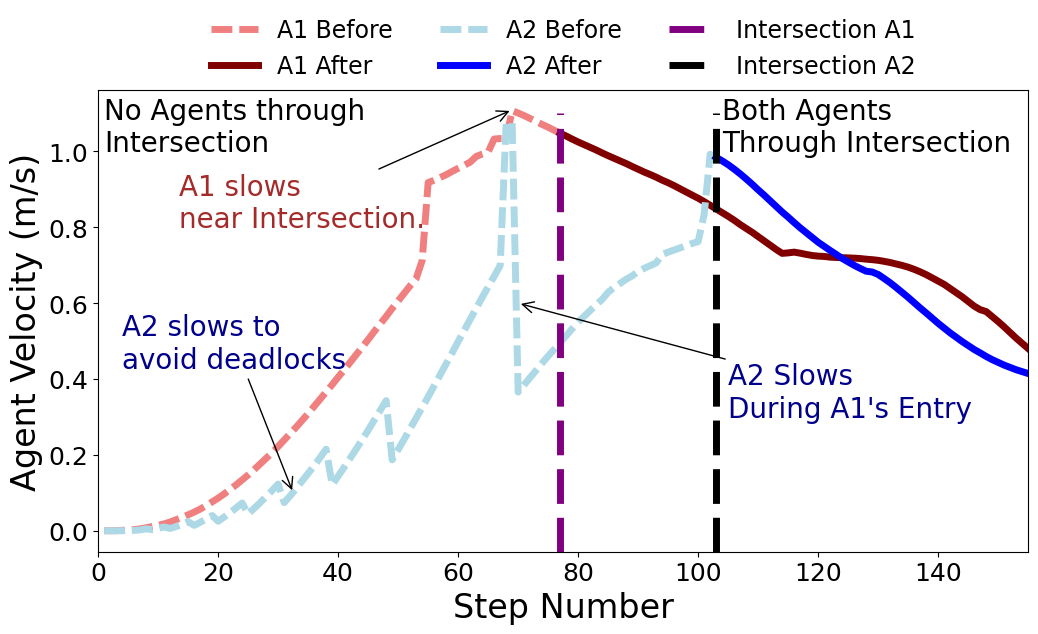}
        \caption{Agent velocities before and after intersection}
        \label{velocity_int}
    \end{subfigure}
    \caption{Velocity profiles of both robots using \framework{} framework, shown before and after navigating through a doorway and an intersection, highlighting the velocity adjustments made to ensure safe and deadlock-free passage.}
    \label{velocity_map}
    \vspace{-10pt}
\end{figure}

Table~\ref{table_combined} provides a quantitative summary of performance metrics. On average, \framework{} achieves zero collisions and a 100\% success rate across 50 runs, significantly outperforming all baselines and ablations. While \framework{} is not the most agile, an expected trade-off given its prioritization of safe and deadlock-free navigation, its reduction in speed is minimal when compared to other methods. Notably, \framework{} achieves the smoothest navigation through the doorway, demonstrating its ability to balance safety and liveness while maintaining smooth, natural robot movement.

\paragraph{Intersection SMG}
     \begin{figure}[t]
        \centering
        \begin{subfigure}[b]{0.49\linewidth}
            \includegraphics[width=0.85\textwidth]{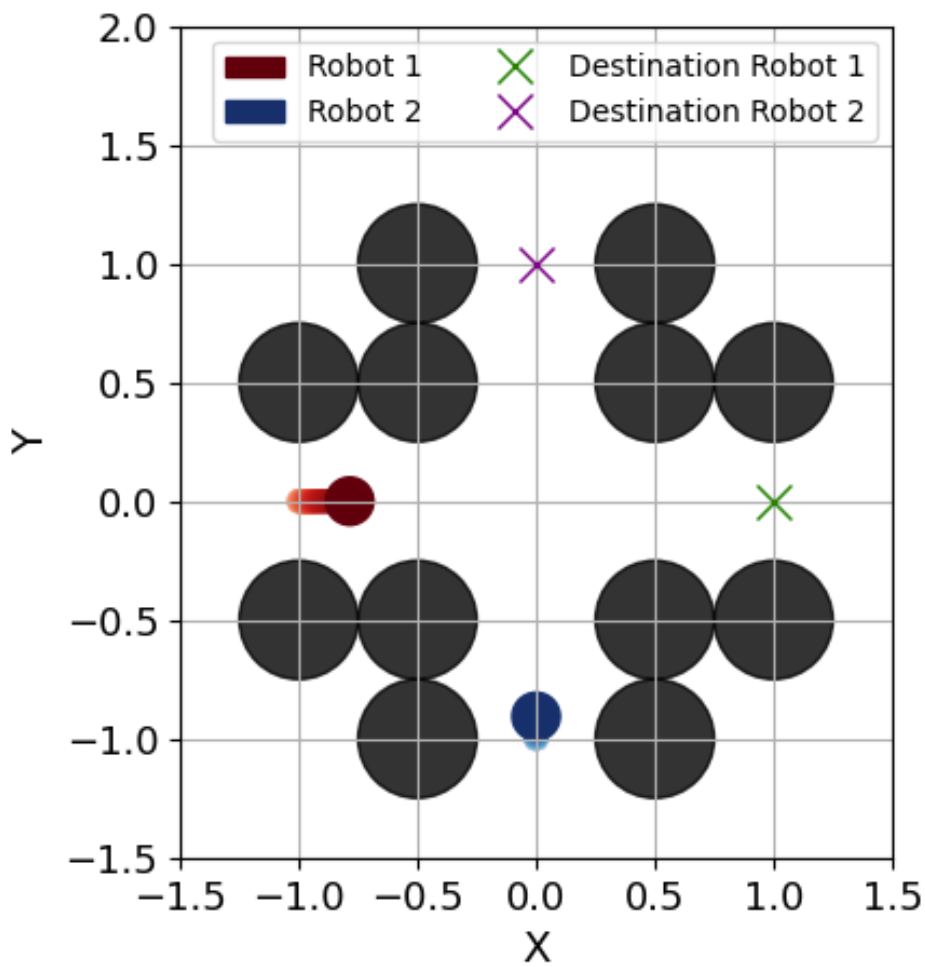}
            \caption{Robots moving to intersection.}
            \label{fig:1a}
        \end{subfigure}
        \hfill
        \begin{subfigure}[b]{0.49\linewidth}
            \includegraphics[width=0.85\textwidth]{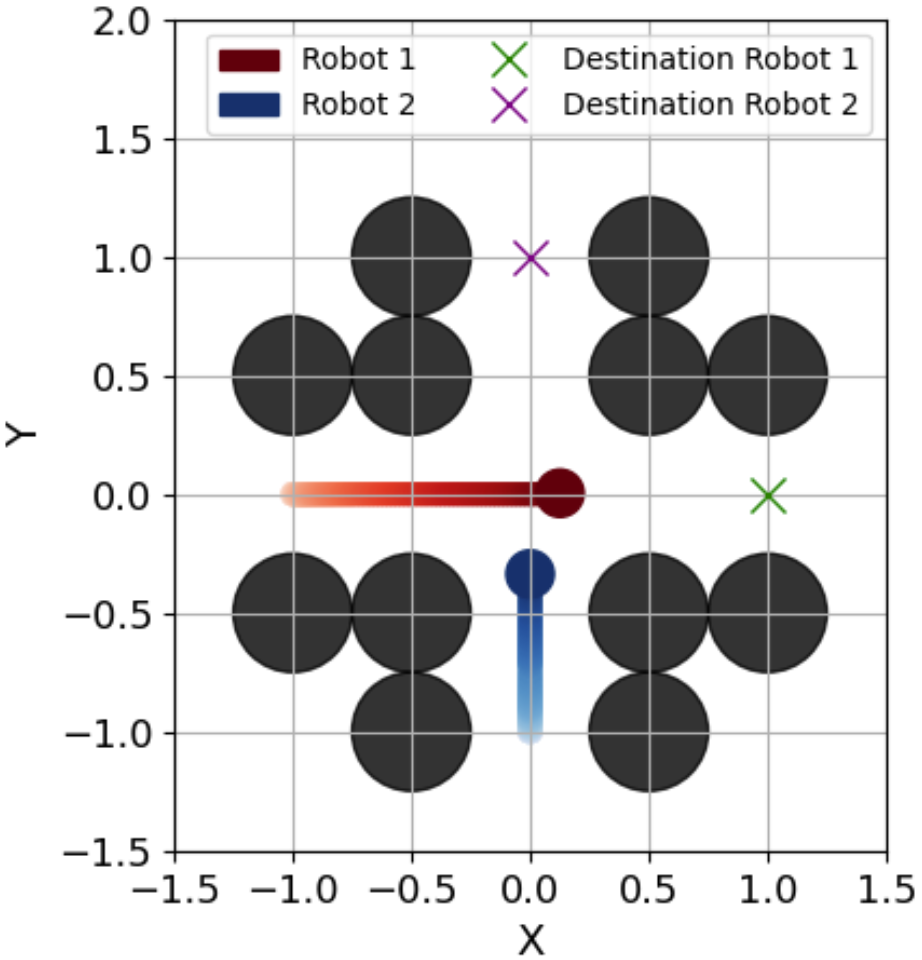}
            \caption{Robot 1 proceeds first.}
            \label{fig:1b}
        \end{subfigure}
        \hfill
        \caption{Intersection: Robot Multi-Agent Trajectories with Liveness}
        \label{liveness_intersection}
    \end{figure}

    \begin{figure}[t]
        \centering
        \begin{subfigure}[b]{0.49\linewidth}
            \includegraphics[width=0.85\textwidth]{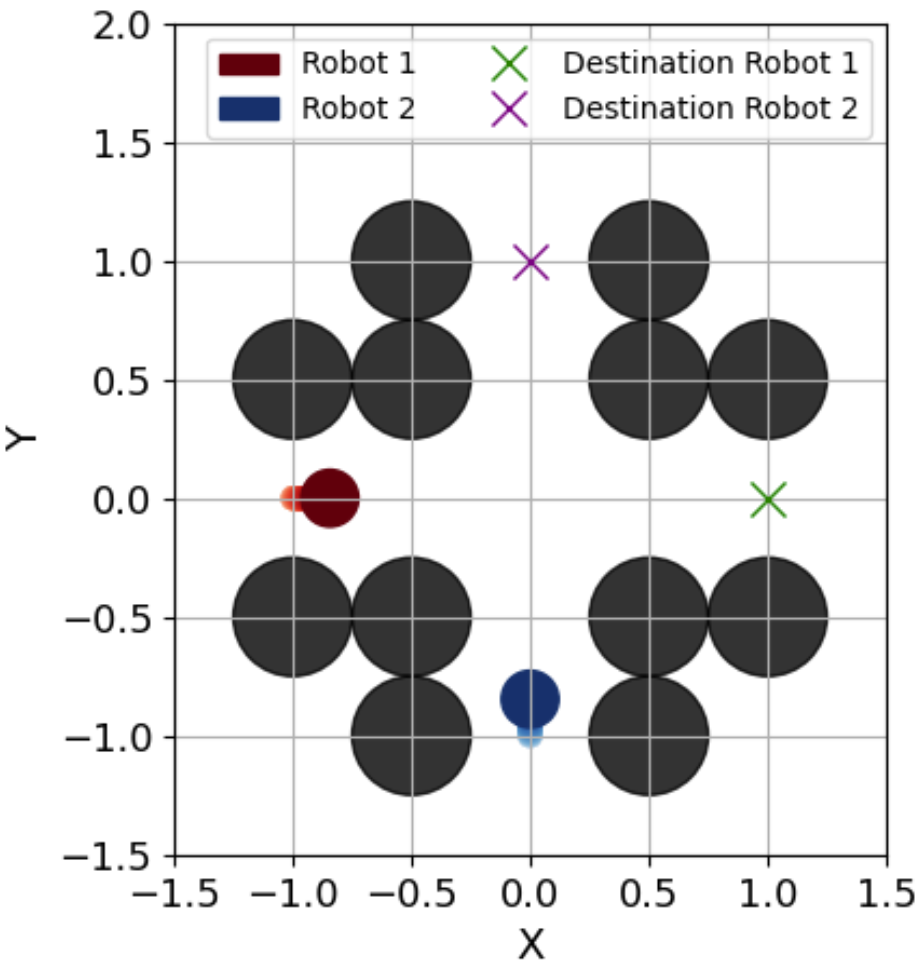}
            \caption{Robots moving to intersection.}
            \label{fig:1a}
        \end{subfigure}
        \hfill
        \begin{subfigure}[b]{0.49\linewidth}
            \includegraphics[width=0.85\textwidth]{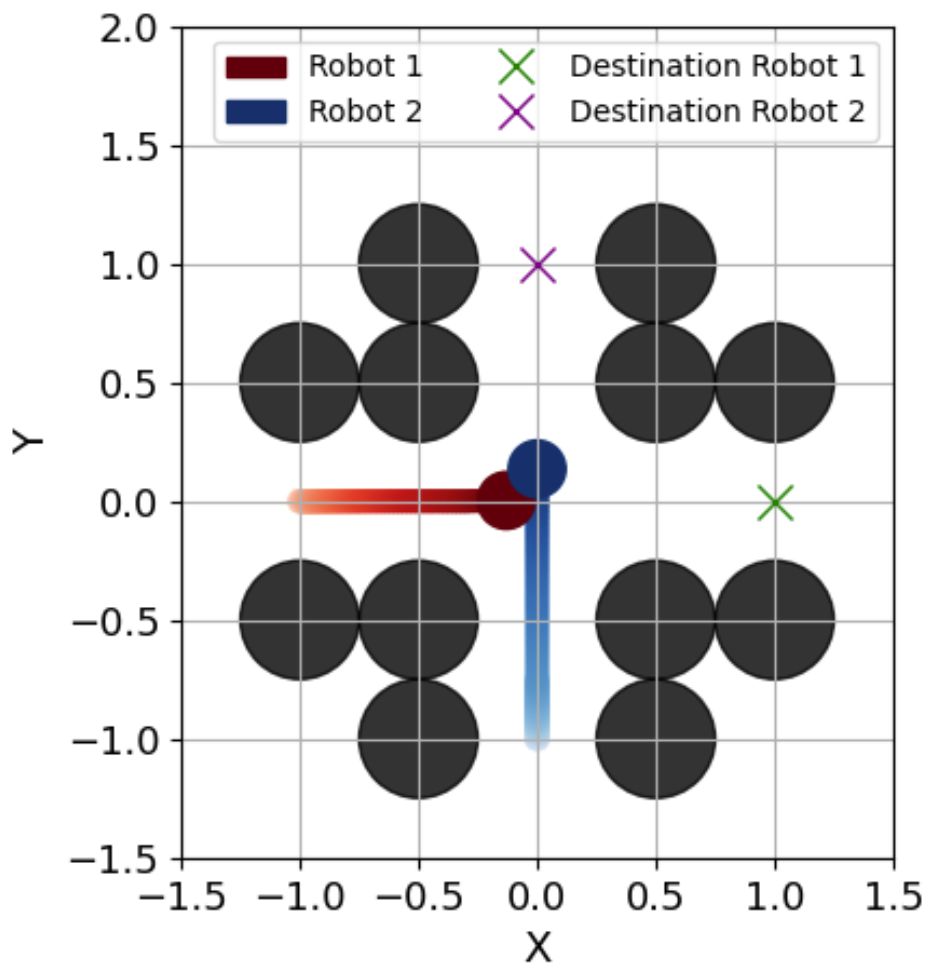}
            \caption{Robots collide in intersection.}
            \label{fig:1b}
        \end{subfigure}
        \hfill
        \caption{Intersection: Multi-Agent Trajectories, No Liveness}
        \label{noliveness_intersection}
    \end{figure}

   \begin{figure}[t]
        \centering
        \begin{subfigure}[b]{0.49\linewidth}
            \includegraphics[width=0.85\linewidth]{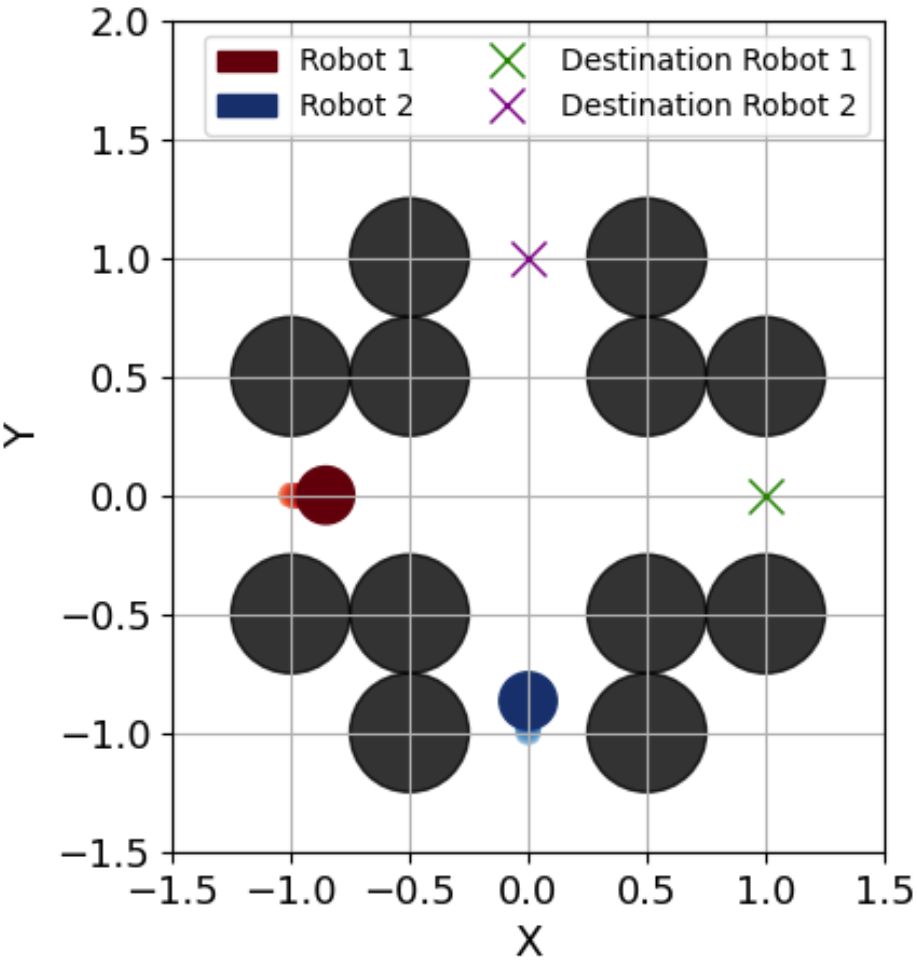}
            \caption{Robots moving towards the intersection.}
            \label{fig:1a}
        \end{subfigure}
        \begin{subfigure}[b]{0.49\linewidth}
            \includegraphics[width=0.85\linewidth]{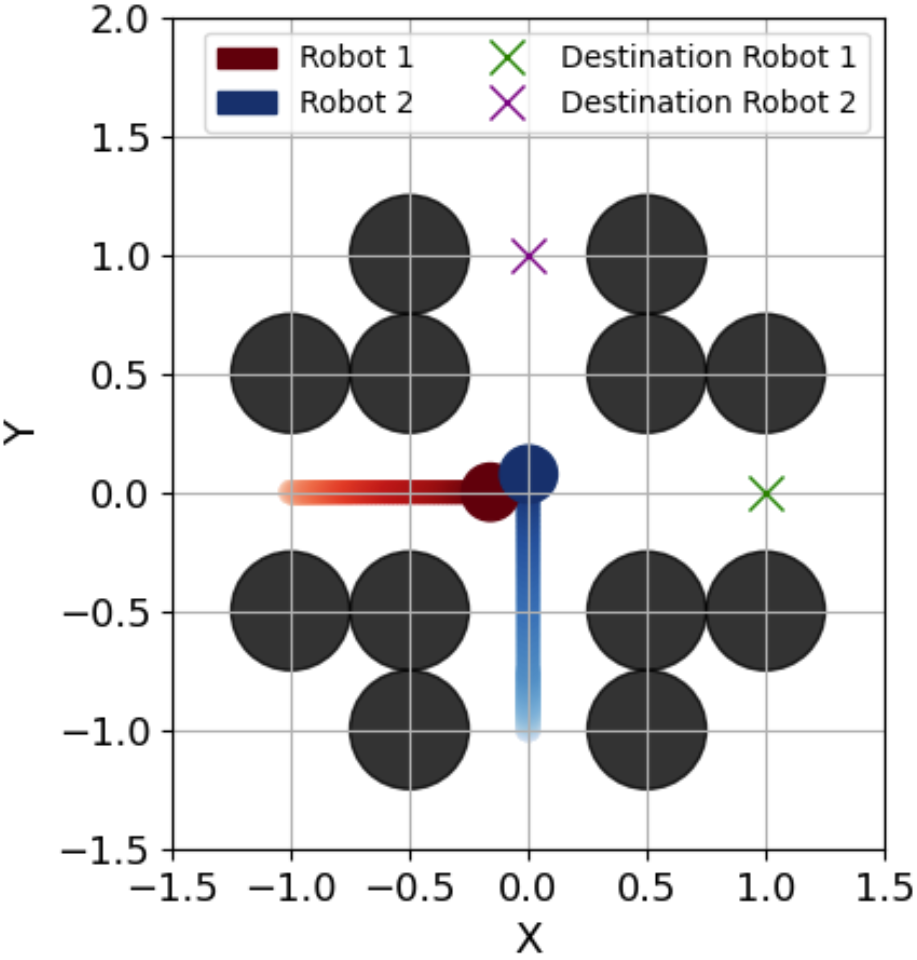}
            \caption{Robot 1 collides with dynamic obstacle.}
            \label{fig:1b}
        \end{subfigure}    
        \caption{Intersection: Single Agent Robot Trajectory, with Dynamic Obstacle Representing Robot 2}
        \label{hybrid_int}
    \end{figure}

We also evaluate \framework{} in an intersection scenario. As non-cooperative robots lack the centralized regulation of human traffic systems, \framework{} is essential for enabling safe, deadlock-free navigation. 





When \framework{} is applied (Figure \ref{liveness_intersection}), the robots successfully navigate the intersection without collisions. Upon detecting potential symmetry early in their approach, Agent 2 reduces its speed, allowing Agent 1 to pass through first, ensuring safe traversal.
In contrast, LoL leads to robots colliding in the intersection, as shown in Figure \ref{noliveness_intersection}, and a single-agent Depth-CBF also leads to collision, as shown in Figure \ref{hybrid_int}.


Figure \ref{int_distance_cbf} illustrates how \framework{} maintains safe separation, and Figure \ref{int_liveness_cbf} shows how the deadlock avoidance algorithm periodically proactively adjusts robot velocities when the liveness value drops below 0.3, preventing potential conflicts, particularly in the early steps where deadlocks are most likely to occur. Figure \ref{velocity_int} shows the velocity profiles of both robots before and after the intersection, showcasing our velocity perturbation's effectiveness. 

Table~\ref{table_combined} shows that identical to the Doorway scenario, \framework{} consistently achieves zero collisions and a 100\% success rate, outperforming baselines and ablations methods. \framework{} successfully prevents deadlocks in the intersection by detecting symmetry early and dynamically adjusting velocities, enabling robots to navigate through the intersection one after the other.





\subsection{Baseline Comparisons}

\subsubsection{ORCA}
ORCA formulates collision avoidance as a set of linear constraints based on the concept of velocity obstacles and solves for an optimal velocity using linear programming. However, ORCA tends to be overly conservative in cluttered and constrained environments. In highly symmetric scenarios such as our SMGs, agents must apply increasingly large corrective actions to avoid collisions. This excessive correction can actually restrict the entire constrained area, preventing the robots from reaching their destinations without violating their safety constraints and causing collisions. As shown in Figure~\ref{orca_traj} and~\ref{orcaint_traj}, replacing our liveness algorithm with ORCA in our framework results in collisions in both doorway and intersection scenarios.

\begin{figure}[t]
        \centering
        \begin{subfigure}{0.325\columnwidth}
            \includegraphics[width=\textwidth]{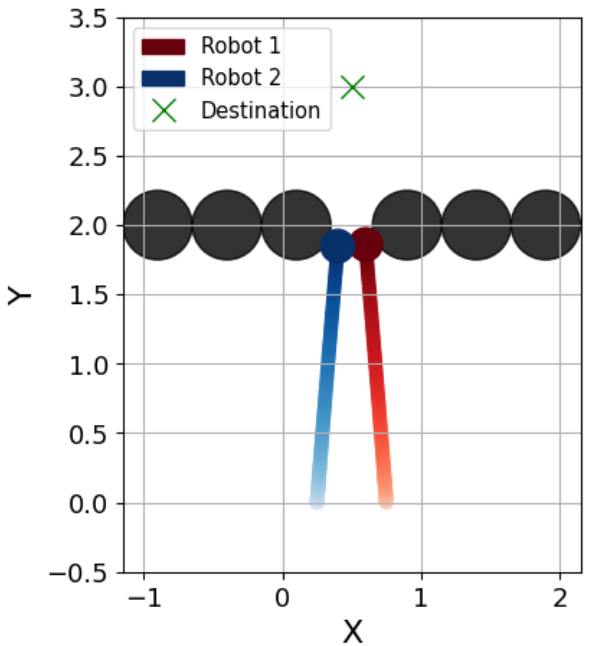}
            \caption{ORCA~\cite{orca2020}.}
            \label{orca_traj}
        \end{subfigure}
        \begin{subfigure}{0.325\columnwidth}
            \includegraphics[width=\textwidth]{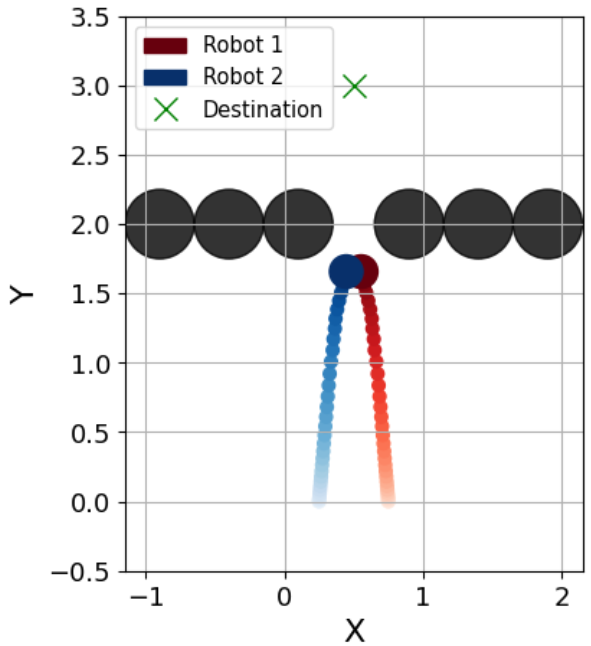}
            \caption{MPC-CBF~\cite{zeng2021mpc}.}
            \label{mpc-cbf_traj}
        \end{subfigure}
        \begin{subfigure}{0.325\columnwidth}
            \includegraphics[width=\textwidth]{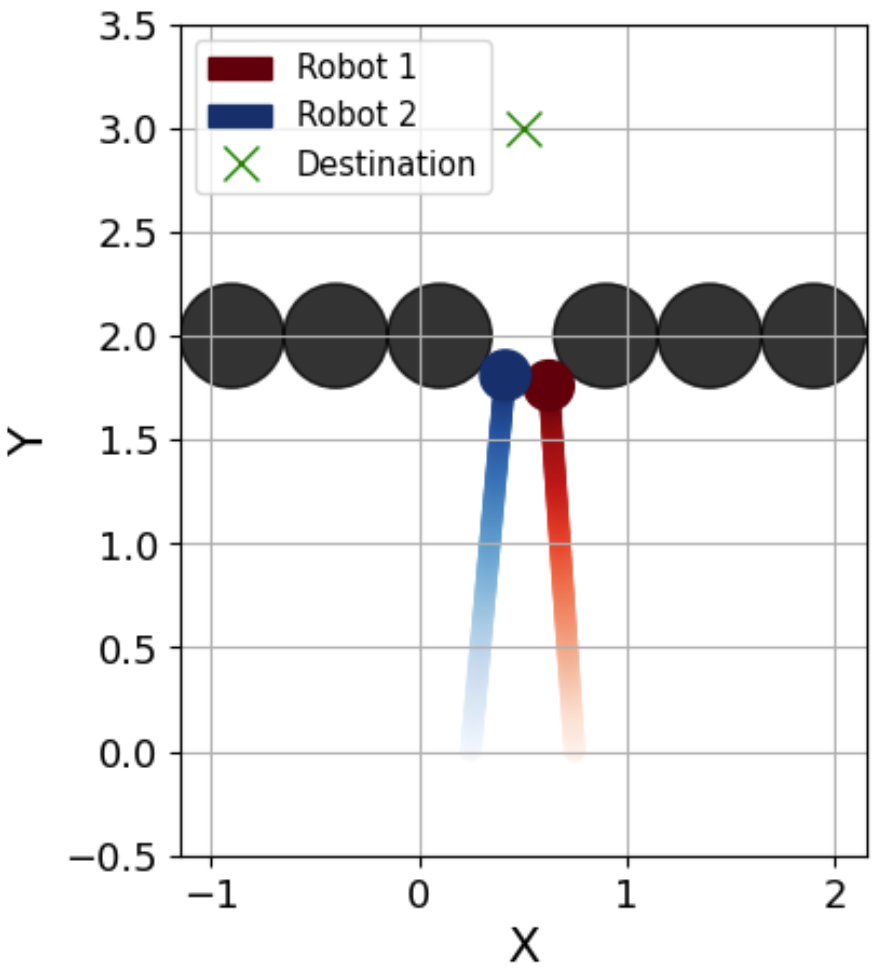}
            \caption{MPNet~\cite{mpnet2020}.}
            \label{mpnet_traj}
        \end{subfigure}
        \caption{Baselines: Multi-Agent Robot Trajectories for ORCA, MPC-CBF, and MPNet in doorway scenario. Collisions occur for all three.}
        \label{baseline_trajectories}
        \vspace{-10pt}
    \end{figure} 

\begin{figure}[t]
        \centering
        \begin{subfigure}{0.325\columnwidth}
            \includegraphics[width=\textwidth]{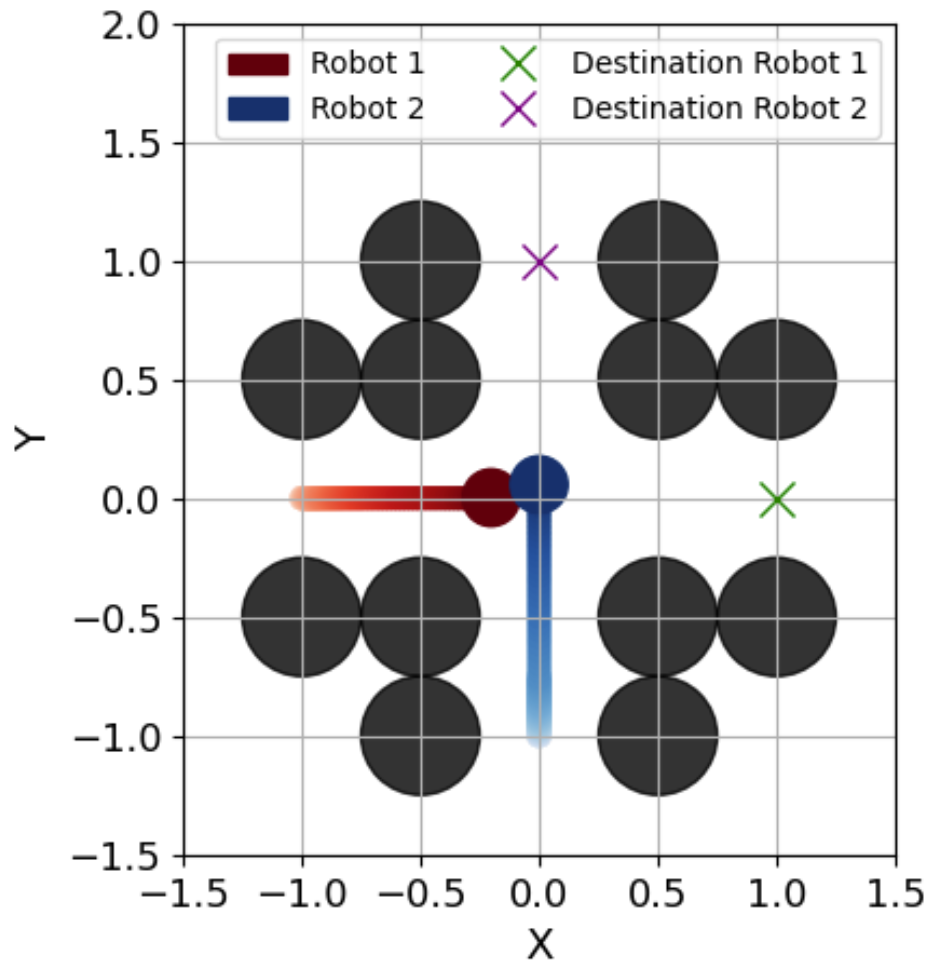}
            \caption{ORCA~\cite{orca2020}.}
            \label{orcaint_traj}
        \end{subfigure}
        \begin{subfigure}{0.325\columnwidth}
            \includegraphics[width=\textwidth]{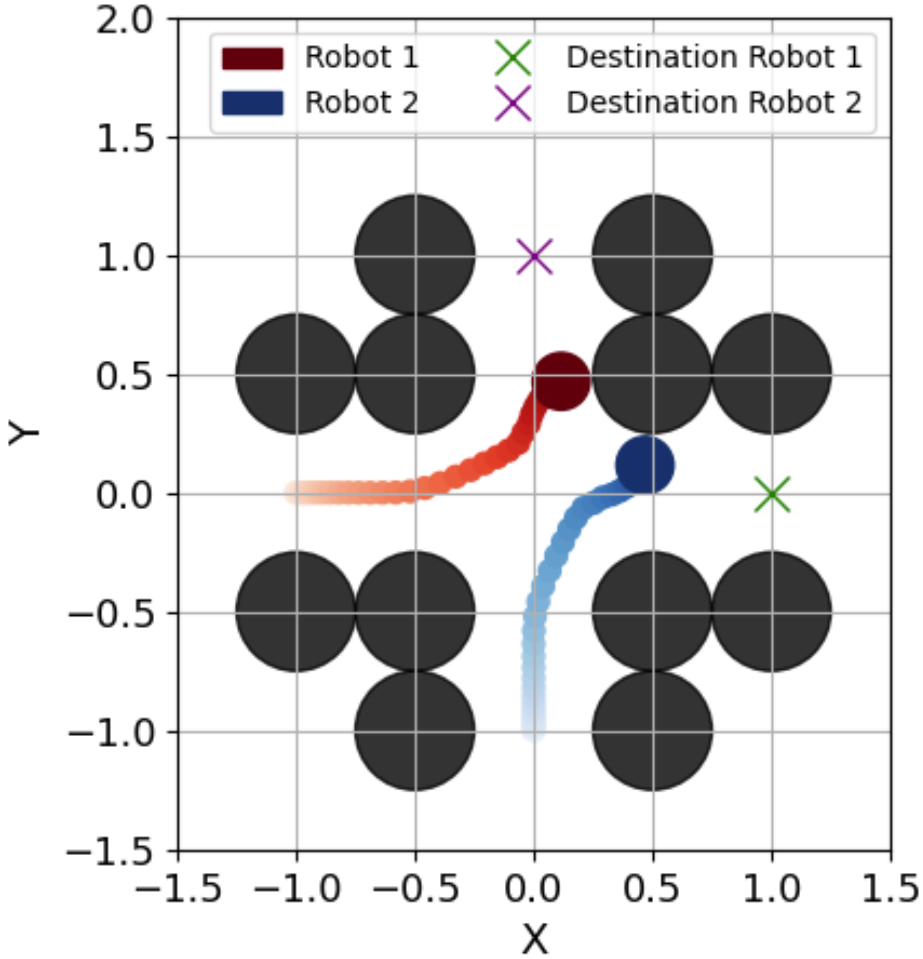}
            \caption{MPC-CBF~\cite{zeng2021mpc}.}
            \label{mpc-cbfint_traj}
        \end{subfigure}
        \begin{subfigure}{0.325\columnwidth}
            \includegraphics[width=\textwidth]{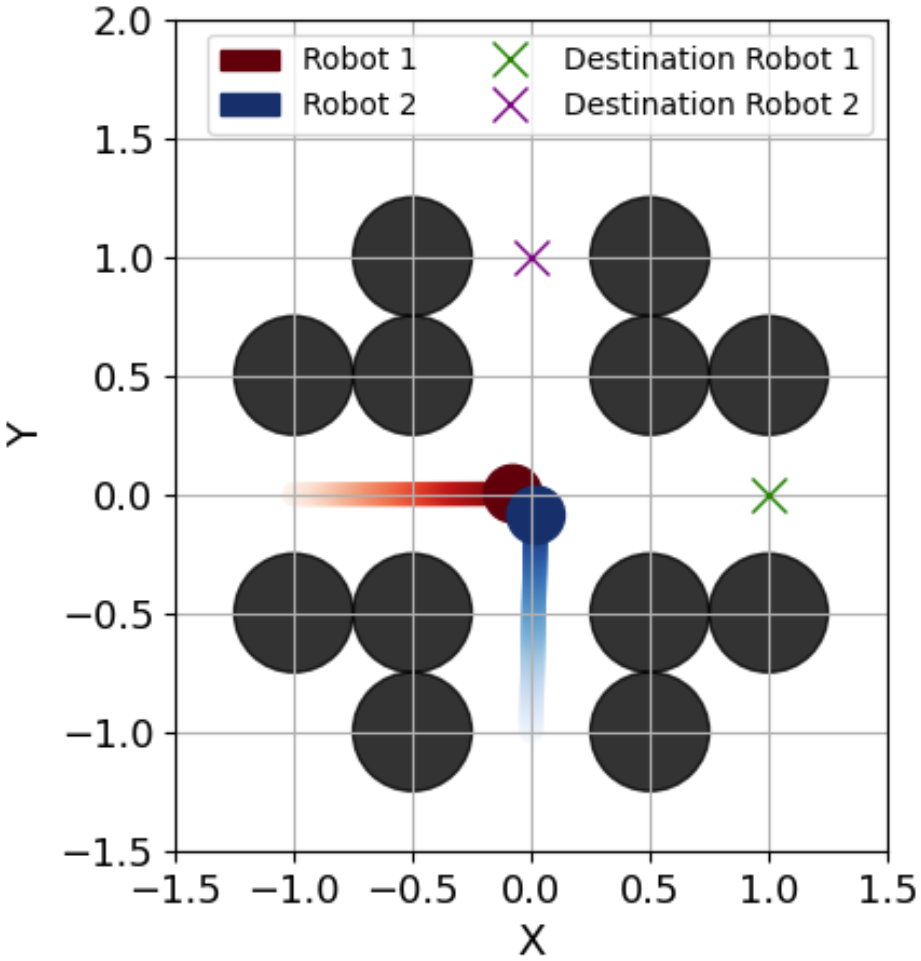}
            \caption{MPNet~\cite{mpnet2020}.}
            \label{mpnetint_traj}
        \end{subfigure}
        \caption{Baselines: Multi-Agent Robot Trajectories for ORCA, MPC-CBF, and MPNet in intersection scenario. Each lead to either collision or deadlock.}
        \label{baselineint_trajectories}
        \vspace{-10pt}
    \end{figure}

\subsubsection{MPC-CBF}
When applied to our SMGs, MPC-CBF is not successful, leading to robot collisions in the doorway scenario, as shown in Figure~\ref{mpc-cbf_traj}, and a deadlock in the intersection scenario, as shown in Figure~\ref{mpc-cbfint_traj}. This failure can be attributed to MPC-CBF enforcing safety constraints based on constant parameters, which do not allow for any adaptation on how strict constraints are in being followed. Thus, in perfectly symmetrical cases, as the SMGs we test on, SMG-CBF lacks any mechanism to resolve deadlocks. In contrast, \framework{} naturally breaks symmetry through dynamic velocity adjustments, allowing agents to avoid deadlocks in real time.

\subsubsection{MPNet}
While MPNet efficiently finds paths in single-agent scenarios, its performance degrades in multi-agent settings. This limitation likely stems from MPNet being trained on static environments, leading to potential overfitting to those conditions and reduced adaptability to dynamic obstacles. As a result, its motion planning can lead to unintended collisions, in constrained SMG environments like doorways and intersections, as shown in Figures \ref{mpnet_traj} and \ref{mpnetint_traj}.

The results highlight the limitations of single-agent Depth-CBF, LoL, and other state-of the-art baselines in constrained multi-agent environments, while \framework{} overcomes these challenges with its novel Universal Safety-Liveness Certificate. By ensuring smooth, safe, and deadlock-free navigation,  \framework{} demonstrates scalability and robustness for real-world multi-agent coordination in dynamic, constrained environments.

\section{CONCLUSION}
\label{Conclusion}

In this work, we introduced a multi-agent navigation framework that synthesizes Control Barrier Functions (CBFs) over point cloud data for safe and deadlock-free navigation in cluttered environments. The proposed method incorporates minimal velocity perturbations to resolve potential deadlocks. The results show that our approach successfully mitigates collisions and deadlocks while maintaining smooth trajectories for multiple agents in complex environments.

Future work could explore extending this framework to scenarios with more agents. While our framework effectively handles two-agent scenarios, its performance in environments with a significantly larger number of agents remains unexplored. Additionally, our current approach may face some challenges in environments where available point cloud information is sparse or incomplete. An extension could be to explore image-based perception as an alternative input modality. Investigating the integration of complimentary or backup modalities could further enhance adaptability, enabling more reliable navigation in real-world applications. 







\bibliographystyle{IEEEtran} 

\end{document}